\documentclass{article}

\usepackage[preprint]{neurips_2026}

\usepackage[utf8]{inputenc} 
\usepackage[T1]{fontenc}    
\usepackage{hyperref}       
\usepackage{url}            
\usepackage{xurl}           
\usepackage{booktabs}       
\usepackage{amsfonts}       
\usepackage{nicefrac}       
\usepackage{microtype}      
\usepackage{xcolor}         
\usepackage{mathtools}
\usepackage{xspace}        
\usepackage{wrapfig}
\usepackage{amsmath}
\usepackage{multirow}
\usepackage{tabularx,booktabs}
\usepackage{enumitem}
\usepackage{graphicx}  
\usepackage{subcaption}
\usepackage[table]{xcolor}
\usepackage{array}
\usepackage{xurl}

\definecolor{uclablue}{RGB}{39, 116, 174}
\definecolor{bigaired}{RGB}{156, 0, 0}
\definecolor{mildblue}{RGB}{31,119,180}

\usepackage{tcolorbox}
\tcbuselibrary{breakable}
\definecolor{customyellow}{HTML}{FFFACD}

\usepackage[dvipsnames]{xcolor}
\definecolor{refinegreen}{RGB}{0, 128, 75}
\definecolor{scoregreen}{RGB}{34, 139, 34}

\usepackage{amsthm}         
\newtheorem{theorem}{Theorem}[section]
\newtheorem{proposition}[theorem]{Proposition}

\newtheorem{remark}[theorem]{Remark}

\usepackage{booktabs}
\usepackage[table]{xcolor}
\usepackage{array}
\usepackage{amsmath}
\definecolor{GRPOrow}{gray}{0.96}
\definecolor{FlowRLrow}{RGB}{225,236,255}
\definecolor{FlowBlue}{RGB}{80,120,210}
\definecolor{GRPOGray}{gray}{0.35}

\newcommand{\method}{\texttt{GFlowRL}\xspace}

\newcommand{\cmark}{\textcolor{green!60!black}{\checkmark}} 
\newcommand{\xmark}{\textcolor{red!80!black}{\times}}   

\title{GFlowRL: Scaling Distribution-Matching RL to Large Language Models}

\author{
  Xiaodong Liu$^{\dagger\ddagger*}$~
  Michael Xu$^{*}$~
  Jack W.~Stokes~
  Paul Smolensky~
  Doug Burger~
  Jianfeng Gao$^{\ddagger}$ \\
  Microsoft Research \\
  \vspace{2pt}
  \footnotesize
  $^{\dagger}$Project lead. \quad
  $^{\ddagger}$Corresponding authors. \quad 
  $^{*}$Equal contribution. \quad
}

\begin{document}

\maketitle

\begin{abstract}
Generative Flow Networks (GFlowNets) offer a promising alternative to reward-maximizing reinforcement learning (RL) for learning large reasoning models.
GFlowNets encourage diverse reasoning paths by matching reward distributions rather than collapsing to dominant modes. 
Recent work shows promise on math and code benchmarks, but scaling GFlowNet-style RL to modern post-training pipelines remains difficult. In the regimes that matter for LLM post-training, model size, rollout horizon, reward noise, and distributed-systems complexity all increase together, and a learned prompt-conditional partition function can become a source of gradient instability and engineering overhead rather than a useful normalizer. 
We address these challenges through systematic analysis to reveal what components are truly necessary for GFlowNet-based RL at scale. Our main finding is that the learned partition function, previously treated as essential, can be replaced by an in-batch Monte Carlo estimate computed from the rollout group already required for training. Based on this observation, we propose \method, a streamlined GFlowNet-style RL algorithm that removes the auxiliary partition network entirely while preserving the reward-distribution-matching objective. The recipe is completed by two stabilizers: importance-sampling correction for rollout/trainer drift, and asymmetric flow-gap clipping for outlier residuals.
\method improves both stability and performance compared to prior work. It exceeds all counterparts on math, code, and adversarial red-teaming benchmarks, reaching a Codeforces rating of \textit{2048} at the 14B scale (within $25$ Elo of o3-mini) and attaining the highest average ASR@1 on AdvBench and HarmBench, outperforming the previous SOTA multi-turn attacker in a regime where FlowRL, a previous GFlowNet-style RL method, diverges. The same recipe transfers to all evaluated MoE configurations up to 235B parameters, where FlowRL again fails to converge. To our knowledge, \method is the first GFlowNet-style RL algorithm to scale stably across both dense and sparse architectures for large reasoning models. Code will be released at \url{https://github.com/microsoft/gflowrl}.
\end{abstract}

\section{Introduction}
\label{sec:intro}
\vspace{-5pt}
Reinforcement learning (RL) has become the defining post-training paradigm for the current generation of large language models. The most capable reasoning systems deployed today, including OpenAI o1~\citep{openai_o1_2024} and o3~\citep{openai_o3_2025}, DeepSeek-R1~\citep{guo2025deepseek}, and Gemini~\citep{Gemini25}, all rely on RL as a central component of their post-training pipelines. The effectiveness of this stage directly determines model performance on high-value tasks such as mathematical reasoning, competitive programming, scientific problem solving, and agentic decision making. As models scale and reasoning chains grow longer, the choice of RL algorithms has become a first-order architectural decision.

The dominant RL approaches, including PPO~\citep{schulman2017proximal}, GRPO~\citep{shao2024deepseekmath}, OMD~\citep{team2025kimi} and their variants, are fundamentally reward maximizing. They concentrate probability mass on high reward modes, often collapsing solution diversity~\citep{xu2025echoes}. In chain-of-thought reasoning, where multiple valid solution paths are essential for generalization and robustness, this behavior introduces a structural limitation. Existing mitigations, such as entropy bonuses~\citep{shen2025entropy}, adaptive clipping, and token-level exploration, treat diversity~\citep{yu2025dapo} as a post hoc regularizer while leaving the core reward maximizing objective unchanged.

Generative Flow Networks (GFlowNets)~\citep{bengio2023gflownet,bengio2021flow,hu2024amortizing} offer a fundamentally different paradigm: instead of maximizing the expected reward, they learn policies that sample trajectories \emph{in proportion to their reward}. When applied to LLM post-training, this distribution-matching perspective directly addresses mode collapse, encouraging coverage of all high-reward reasoning paths rather than concentrating on one. Recent approaches such as FOR~\citep{yu2025flow} and FlowRL~\citep{zhu2026flowrl} are among the first to introduce this idea into LLM reasoning, adopting a trajectory balance objective with a learned approximation to the partition function $Z_\phi(\mathbf{x})$. In FlowRL, $Z_\phi(\mathbf{x})$ is parameterized as a three-layer MLP on top of the prompt’s final hidden state, whereas in FOR it is represented as a scalar shared across all prompts.

\textit{The scaling problem.}
 While FlowRL shows the promise of GFlowNet-style training, deploying it in the regimes that matter for modern LLM post-training remains challenging. Even on dense backbones, training must operate under a short optimization horizon, long reasoning rollouts, large pretrained policies, and tight distributed-systems budgets. In this setting, jointly training a new initialized partition model alongside the policy becomes increasingly brittle as scale grows: gradient spikes become more frequent, the auxiliary module adds optimizer and synchronization overhead, and the learned normalizer lags behind the much larger pretrained policy. MoE routing~\citep{guo2025deepseek,achiam2023gpt,grattafiori2024llama3herdmodels,team2025kimi,Gemini25} amplifies these issues further by introducing nondeterminism and implicit off-policy mismatch between sampling and training~\citep{thinkingmachines2025nondeterminism,yao2025offpolicy_mismatch}, but MoE is best viewed as the hardest stress test of a more general scaling bottleneck rather than the root cause itself.

\textit{The root cause.} We attribute these failures to a key design choice: the learnable partition function $Z_\phi$. In LLM post-training, there is a fundamental asymmetry between the two components being optimized. The policy is initialized from a pretrained model with billions of well-trained parameters and requires only modest updates, whereas $Z_\phi$ is randomly initialized and must learn a complex quantity from scratch within the same limited training steps. As a result, $\log Z_\phi$ behaves largely as a source of noise for much of training, causing the trajectory balance objective to be dominated by variance rather than the reward signal. Section~\ref{subsec:analysis} provides a detailed analysis of the gradient behavior and training dynamics. Empirically, we have observed two consistent patterns: (1) replacing $Z_\phi$ with a random noise yields comparable performance, suggesting that it contributes little to modeling; and (2) FlowRL exhibits gradient norms orders of magnitude larger than GRPO, leading to significant optimization instability without bringing any visible gains in accuracy.

Based on this analysis, we propose \textbf{\method}, a streamlined GFlowNet-based RL algorithm that does not use the learned partition function. In its place, we use a simple yet effective in-batch Monte Carlo estimate of $\log Z$, estimated directly from the rollout group already sampled during GRPO for each prompt. This simple change has three consequences: (i) the gradient norms remain the same scale as that of GRPO, restoring training stability; (ii) the auxiliary network and its optimizer state in distributed training are all removed; and (iii) the algorithm uses a nearly identical infrastructure of GRPO, while retaining the reward-distribution-matching semantics of trajectory balance. 
The recipe is completed by two stabilizers: importance-sampling correction for rollout/trainer drift, and asymmetric flow-gap clipping for outlier residuals. To our knowledge, \method is the \textit{first GFlowNet-based RL algorithm to scale stably across both dense and sparse architectures}.

In summary, our contributions are threefold:
\begin{enumerate}[leftmargin=*, topsep=0pt, noitemsep]
    \item We identify the use of the learned partition function $Z_\phi$ as a key source of instability in GFlowNet-based RL training for LLMs. We provide empirical evidence that $Z_\phi$ offers almost no benefit for post-training 
    and instead can make optimization unnecessarily challenging.
    \item We propose \method, where $Z_\phi$ is replaced with an in-batch Monte Carlo estimate, 
    thus removing the auxiliary network and its associated distributed overhead while retaining the trajectory balance objective’s fixed point.
    \item We demonstrate stable training across both dense and sparse MoE architectures, including regimes where FlowRL diverges. At the 14B dense scale, \method reaches \textit{2048} Codeforces Elo, exceeding DeepCoder-14B by $+112$, FlowRL-14B by $+144$, and OpenAI's o1 by $+157$. On adversarial red-teaming, \method attains the highest average ASR@1 on AdvBench (\textit{82.5\%}) and HarmBench (\textit{79.5\%}), surpassing the prior SOTA multi-turn attacker in a noisy-reward regime where FlowRL diverges. Scaling to MoE, \method reaches \textit{1999} Codeforces Elo on Qwen3-30B-A3B with only 3B active parameters, and the same recipe transfers to Qwen3-235B-A22B on math.
    \end{enumerate}
\vspace{-4pt}

\section{Preliminaries}
\label{sec:prelim}
\vspace{-5pt}
We formulate chain-of-thought reasoning as a conditional generation task. Given a prompt $\mathbf{x} \in \mathcal{X}$ (e.g., a math problem), the policy $\pi_\theta(\mathbf{y} \mid \mathbf{x})$ autoregressively generates a response $\mathbf{y} = (y_1, y_2, \dots, y_T)$ containing intermediate reasoning and a final answer. A verifier provides a scalar reward $r(\mathbf{x}, \mathbf{y}) \in \mathbb{R}$. For example, the reward may be a binary correctness signal for math~\citep{shao2024deepseekmath} or unit-test pass rates for code~\citep{jain2024livecodebench}.

\noindent\textbf{Policy notation.}
A symbol $\pi(\mathbf{y} \mid \mathbf{x})$ denotes a policy, i.e., a conditional distribution over responses $\mathbf{y}$ given prompt $\mathbf{x}$. Following~\citep{schulman2017proximal, shao2024deepseekmath, yao2025offpolicy_mismatch}, $\pi_\theta$ is the trainable policy, $\pi_{\mathrm{ref}}$ the frozen reference (e.g., the pretrained model), and $\pi_{\mathrm{old}}$ the rollout policy used to sample the current group; $\pi_{\mathrm{old}}$ is a snapshot of $\pi_\theta$ used for rollout generation while $\pi_\theta$ is updated against those rollouts.

\noindent\textbf{Reward-maximization.} Typical reward-maximizing objectives such as PPO~\citep{schulman2017proximal} and GRPO~\citep{shao2024deepseekmath} optimize
\begin{equation}
    \mathcal{J}_{\text{RM}}(\theta) = \mathbb{E}_{\mathbf{x} \sim \mathcal{D},\, \mathbf{y} \sim \pi_\theta(\cdot \mid \mathbf{x})} \big[\, r(\mathbf{x}, \mathbf{y}) \,\big],
\label{eq:reward_max}
\end{equation}
which tends to concentrate the probability mass on the single highest-reward mode of the response distribution (Figure~\ref{fig:step500}(a)) and leads to mode collapse and reduced solution diversity~\citep{xu2025echoes}.

\noindent\textbf{Reward distribution matching.}
An alternative is to sample responses in proportion to their reward rather than maximizing it. Define the target distribution as the softmax
\begin{equation}
p^\star(\mathbf{y} \mid \mathbf{x}) = \frac{1}{Z(\mathbf{x})} \exp\!\left({\beta} r(\mathbf{x}, \mathbf{y}) \right), \qquad Z(\mathbf{x}) = \sum_{\mathbf{y} \in \mathcal{Y}} \exp\!\left( {\beta r(\mathbf{x}, \mathbf{y})} \right),
\label{eq:target_dist}
\end{equation}
where $\beta > 0$ is an inverse temperature controlling the sharpness of the distribution, and $Z(\mathbf{x})$ is the (intractable) partition function. Minimizing the reverse KL divergence $D_{\mathrm{KL}}(\pi_\theta \,\|\, p^\star)$ encourages the policy to cover all high-reward modes proportionally, rather than collapsing onto one~\citep{bengio2023gflownet,bengio2021flow}. Directly minimizing this KL is infeasible due to the intractable $Z(\mathbf{x})$, motivating surrogate objectives that enforce the same fixed point while learning $Z$ as a side product.

\noindent\textbf{The trajectory balance objective in LLM reasoning.} FOR~\citep{yu2025flow} and FlowRL~\citep{zhu2026flowrl} both instantiate this distribution-matching principle through the trajectory balance (TB) objective~\citep{malkin2022tb} from GFlowNets. For a prompt $\mathbf{x}$ and sampled response $\mathbf{y}$, the TB loss is
\vspace{-2pt}
\begin{equation}
\mathcal{L}_{\text{TB}}(\theta, \phi; \mathbf{x}, \mathbf{y}) = \left( \log Z_\phi(\mathbf{x}) + \log \frac{\pi_\theta(\mathbf{y} \mid \mathbf{x})}{\pi_{\mathrm{ref}}(\mathbf{y} \mid \mathbf{x})} - {\beta} r(\mathbf{x}, \mathbf{y}) \right)^{\!2},
\label{eq:tb_loss}
\end{equation}
where $Z_\phi(\mathbf{x})$ is a \emph{learned} approximation to the true partition function. In FlowRL, $Z_\phi(\mathbf{x})$ is a 3-layer MLP on top of the prompt's final hidden state, whereas in FOR, the partition function is a single shared variable that is used across all prompts.
\vspace{-6pt}
\section{Methodology}
\label{sec:method}
\vspace{-6pt}
\subsection{Why the Learned Partition Function Fails}
\label{subsec:analysis}
While the TB objective is theoretically principled, we argue that its standard instantiation is poorly suited to the LLM post-training regime. The crux is a \textit{learning-horizon mismatch} between the two components of the joint optimization.

\noindent\textbf{The learning-horizon mismatch.}
The policy $\pi_\theta$ is initialized from a pretrained LLM checkpoint with billions of parameters that encode rich linguistic and mathematical priors from trillions of tokens of pretraining~\citep{qwen2.5}. Post-training RL needs only to \textit{refine} this strong prior, and typically runs for a few hundred gradient updates. The partition function $Z_\phi$, in contrast, is a freshly initialized network (e.g., 3-layer MLP as in FlowRL) that must learn a complex, prompt-conditional quantity \textit{from scratch} within the same short window. This asymmetry means that throughout most of training, $\log Z_\phi$ is effectively a random function of the prompt.

\noindent\textbf{Empirical evidence I: the partition function contributes little.}
If the partition function $Z_\phi$ provides a meaningful learning signal, replacing it with noise should degrade performance. We test this by removing $Z_\phi$ entirely and substituting it with random samples drawn from a normal distribution with mean $\mu = 0.5$ and standard deviation $\sigma = 1$. Detailed hyperparameters are listed in Table~\ref{tab:hyperparams}. Surprisingly, the performance remains unchanged and even slightly improves: the random variant achieves $36.19\%$ average accuracy versus $35.61\%$ for the learned partition function. This indicates that $Z_\phi$ does not encode useful information for the policy and is effectively ignored during training.

\begin{table*}[h]
\centering
\vspace{-6pt}
\caption{Gradient norm statistics over 421 training steps.
         $^\dagger$FlowRL statistics include 55 explosion steps ($\geq 10^6$);
         excluding them yields median $= 1.1\times10^3$.}
\label{tab:gradnorm}
\small
\setlength{\tabcolsep}{4pt}
\begin{tabular}{lrrrrrr}
\toprule
\textbf{Method} & \textbf{Min} & \textbf{Max} & \textbf{Mean} & \textbf{Median} & \textbf{Std} & \textbf{Expl. ($\geq 10^6$)} \\
\midrule
\textbf{FlowRL}$^\dagger$ 
& $1.58\times10^{2}$  
& $9.59\times10^{16}$ 
& $3.23\times10^{14}$ 
& $1.14\times10^{3}$  
& $4.89\times10^{15}$ 
& 55 \\
\textbf{GFlowRL}          
& $3.11\times10^{-2}$ 
& $6.18$              
& $6.98\times10^{-2}$ 
& $4.30\times10^{-2}$ 
& $3.08\times10^{-1}$ 
& 0 \\
\textbf{GRPO}             
& $1.68\times10^{-1}$ 
& $5.90$              
& $2.38\times10^{-1}$ 
& $1.99\times10^{-1}$ 
& $3.37\times10^{-1}$ 
& 0 \\
\bottomrule
\end{tabular}
\vspace{-6pt}
\end{table*}
\noindent\textbf{Empirical evidence II: the partition function destabilizes training.}
Although $Z_\phi$ provides no measurable benefit, it introduces severe optimization instability. Gradient norm statistics in Tab~\ref{tab:gradnorm} reveal a stark contrast: while GRPO and GFlowRL maintain well-behaved norms (means of $0.24$ and $0.07$, with maxima below $6.2$), FlowRL exhibits extreme variance, with a mean of $3.2 \times 10^{14}$, a maximum of $9.6 \times 10^{16}$, and a standard deviation of $4.9 \times 10^{15}$ (Fig.~\ref{fig:grad_norm} in Appendix~\ref{app:training})\footnote{FlowRL training logs sourced from \url{https://wandb.ai/xuekaizhu0/FlowRL/runs/pa62rs4x}.}. Notably, $55$ out of $421$ steps produce gradient explosions exceeding $10^6$. This instability arises from the unbounded contribution of $\log Z_\phi(\mathbf{x})$, in training when the network is poorly calibrated, so policy updates are dominated by high-variance signals unrelated to the reward or likelihood.

Together, these results reveal a clear pattern: the partition function neither improves performance nor behaves benignly. It acts as a source of noise that degrades optimization stability without contributing a useful learning signal, motivating a simple question: \textit{if the learned partition function is both unnecessary and harmful, can we replace it with a non-learned estimator while preserving the distribution-matching objective?}

\subsection{GFlowRL}
Our proposal is to replace the learned network $\log Z_\phi(\mathbf{x})$ with an in-batch Monte Carlo estimate computed from the rollouts themselves. For each prompt $\mathbf{x}$, GRPO-style training already samples a group of $G$ rollouts $\{\mathbf{y}^{(1)}, \dots, \mathbf{y}^{(G)}\}$ from the rollout policy $\pi_{\phi_{\text{old}}}$. Inspection of the TB loss (Eq.~\ref{eq:tb_loss}) reveals that, at the optimum, $\log Z(\mathbf{x})$ equals $\beta r(\mathbf{x}, \mathbf{y}) + \log \pi_{\theta_{\text{ref}}}(\mathbf{y} \mid \mathbf{x}) - \log \pi_\theta(\mathbf{y} \mid \mathbf{x})$ for every $\mathbf{y}$ in the support of $p^\star$. This suggests the natural in-batch estimator of $\log Z_t$:
\begin{equation}
\mathcal{Z}_t(\mathbf{x}) \;\coloneqq\; \frac{1}{G} \sum_{i=1}^{G} \left({\beta} {r(\mathbf{x}, \mathbf{y}^{(i)})} + \log \pi_{\theta_{\text{ref}}}(\mathbf{y}^{(i)} \mid \mathbf{x}) - \log \pi_{\phi_{\text{old}}}(\mathbf{y}^{(i)} \mid \mathbf{x}) \right),
\label{eq:batch_logZ}
\end{equation}
which averages the per-sample TB targets across the current rollout group. 
This estimator builds on the log-partition variance loss of~\citet{robust-scheduling}, which bypasses learning $\log Z$ by exploiting a property of the TB optimum~\citep{malkin2022tb}: every trajectory yields the same implicit target, but as shown in Appendix~\ref{app:comp_logs}, this loss barely improves over the backbone in our setting. In their setting, the condition is a computation graph and trajectories are sampled schedules; in ours, the condition is the prompt $\mathbf{x}$ and trajectories are on-policy rollouts, e.g., the $G$ samples drawn per prompt in GRPO. Rather than minimize the variance of these targets across trajectories sharing the same condition, we take their in-batch mean as a stop-gradient baseline that plugs directly into the policy-gradient update. This is a natural fit for LLM post-training, where the rollout group is already the unit of computation.

\paragraph{The \method objective.}
Substituting $\mathcal{Z}_t(\mathbf{x})$ into Eq.~\ref{eq:tb_loss}, the distribution-matching residual for rollout $\mathbf{y}^{(i)}$ takes the form
\begin{equation}
\Delta^{(i)}(\theta) \;=\; \operatorname{sg}\!\left[\mathcal{Z}_t(\mathbf{x})\right] \;+\; \frac{1}{|\mathbf{y}^{(i)}|} \log \frac{\pi_\theta(\mathbf{y}^{(i)} \mid \mathbf{x})}{\pi_{\theta_{\text{ref}}}(\mathbf{y}^{(i)} \mid \mathbf{x})} \;-\; \beta r(\mathbf{x}, \mathbf{y}^{(i)}),
\label{eq:residual_main}
\end{equation}
where $|\mathbf{y}^{(i)}|$ is the response length and $\operatorname{sg}[\cdot]$ is the stop-gradient operator. $\mathcal{Z}_t(\mathbf{x})$ enters only through $\operatorname{sg}[\cdot]$: it serves as a baseline that centers the residual but carries no gradient of its own, eliminating the auxiliary optimization over a learned partition function entirely. Reasoning rollouts can span thousands of tokens, causing instability as log-probabilities scale with length; we normalize by response length to prevent long sequences from dominating the loss~\citep{yu2025dapo,zhu2026flowrl}.

A naive squared-residual objective $\tfrac{1}{G}\sum_i w^{(i)} (\Delta^{(i)})^2$ is still vulnerable to outlier rollouts whose residual is arbitrarily large, inducing destabilizing updates. To control this, we decompose $\Delta^{(i)}(\theta)$ into a \emph{flow gap} evaluated at the rollout policy,
\begin{equation}
g^{(i)} \;\coloneqq\; \operatorname{sg}\!\left[\mathcal{Z}_t(\mathbf{x})\right] \;+\; \frac{1}{|\mathbf{y}^{(i)}|} \log \frac{\pi_{\mathrm{old}}(\mathbf{y}^{(i)} \mid \mathbf{x})}{\pi_{\mathrm{ref}}(\mathbf{y}^{(i)} \mid \mathbf{x})} \;-\; {\beta} r(\mathbf{x}, \mathbf{y}^{(i)}),
\label{eq:flow_gap}
\end{equation}
plus a policy-update term $\tfrac{1}{|\mathbf{y}^{(i)}|}\log\bigl(\pi_\theta(\mathbf{y}^{(i)}\!\mid\!\mathbf{x})/\pi_{\mathrm{old}}(\mathbf{y}^{(i)}\!\mid\!\mathbf{x})\bigr)$ that carries the gradient, so that $\Delta^{(i)}(\theta) = g^{(i)} + \tfrac{1}{|\mathbf{y}^{(i)}|}\log(\pi_\theta/\pi_{\mathrm{old}})$. Following the spirit of trust-region methods~\citep{schulman2017proximal, yu2025dapo}, we clip the flow gap asymmetrically:
\begin{equation}
\tilde{g}^{(i)} \;=\; \operatorname{clip}\!\left(g^{(i)},\, -\epsilon_{\text{low}},\, +\epsilon_{\text{high}}\right).
\label{eq:flow_gap_clip}
\end{equation}
The asymmetric bounds $\epsilon_{\text{low}} < \epsilon_{\text{high}}$ give more room for positive corrections than negative ones. This is especially useful on math reasoning tasks, where correct solutions can be under-sampled early in training and the update should be able to push probability mass upward more aggressively than it pushes already-favored trajectories downward. Substituting $\tilde{g}^{(i)}$ for $g^{(i)}$ in the residual and weighting by importance-sampling ratios yields the \method training loss:
\begin{equation}
\mathcal{L}_{\text{\method}}(\theta; \mathbf{x}) \;=\; \frac{1}{G} \sum_{i=1}^{G} w^{(i)} \left(\tilde{g}^{(i)} \;+\; \frac{1}{|\mathbf{y}^{(i)}|} \log \frac{\pi_\theta(\mathbf{y}^{(i)} \mid \mathbf{x})}{\pi_{\mathrm{old}}(\mathbf{y}^{(i)} \mid \mathbf{x})}\right)^{\!2}.
\label{eq:gflowrl_loss}
\end{equation}
The importance-sampling weight
$w^{(i)}=\min\!\left( \frac{\pi_\theta(\mathbf{y}^{(i)} \mid \mathbf{x})}{\pi_{\mathrm{old}}(\mathbf{y}^{(i)} \mid \mathbf{x})}, \; 1 + \epsilon \right)$
corrects for the mismatch between the rollout policy $\pi_{\mathrm{old}}$ and the current policy $\pi_\theta$ under optimization~\citep{yao2025offpolicy_mismatch}.

\noindent\textit{Why this works.}
When the flow-gap clip is inactive ($\tilde{g}^{(i)} = g^{(i)}$), the residual inside the square in Eq.~\ref{eq:gflowrl_loss} is exactly $\Delta^{(i)}(\theta)$. Thus, without length normalization, the objective preserves the same target fixed point as trajectory balance: any zero-loss self-consistent fixed point satisfies $\pi_\theta(\mathbf{y} \mid \mathbf{x}) \propto \pi_{\theta_{\text{ref}}}(\mathbf{y} \mid \mathbf{x}) \exp(\beta r(\mathbf{x}, \mathbf{y}))$ (Proposition~\ref{prop:optimality}, Appendix~\ref{app:proof}). This result is a characterization of the desired fixed point, not a guarantee that the finite-batch nonconvex optimization dynamics necessarily converges to it.
The clip is a safeguard on outlier rollouts: at the fixed point the flow gap is zero, so the clip is inactive in a neighborhood of the optimum and does not perturb the stationary distribution (Remark~\ref{rem:clip}). Length normalization introduces a mild approximation whose effect vanishes when response lengths are approximately uniform across rollouts, which holds in practice for reasoning tasks (Remark~\ref{rem:length_norm}). The key practical difference is that $\mathcal{Z}_t(\mathbf{x})$ is an estimator whose scale is automatically commensurate with $r/\beta$ and $\log \pi_\theta$, eliminating the gradient explosion observed with the learned $Z_\phi$. Moreover, $\mathcal{Z}_t(\mathbf{x})$ requires no additional parameters, no separate learning rate, and no cross-module synchronization in distributed training.
\vspace{-6pt}
\section{Experiment}
\label{sec:exp}
\vspace{-6pt}

\subsection{Experimental Setup}
\label{subsec:setup}
\vspace{-6pt}
\noindent\textbf{Models, datasets, and metrics.}
We evaluate \method across three reasoning domains. For \textit{math}, we train Qwen2.5-7B/32B~\citep{qwen2.5} on six standard benchmarks (AIME2024/2025~\citep{AIME24,AIME25}, AMC2023~\citep{AMC}, MATH-500~\citep{lightman2023let}, Minerva~\citep{minerva}, Olympiad~\citep{he2024olympiadbench}), reporting Avg@16~\citep{yu2025dapo, zhu2026flowrl}. For \textit{coding}, we train Deepseek-R1-Distilled-Qwen-7B/14B~\citep{guo2025deepseek} on LiveCodeBench~\citep{jain2024livecodebench}, Codeforces~\citep{penedo2025codeforces} (Rating/Percentile, following~\citep{deepcoder2025}), and HumanEval+~\citep{chen2021evaluating}. For \textit{LLM safety}, we use SEMA~\citep{feng2026sema} on AdvBench~\citep{zou2023universal} and HarmBench~\citep{mazeika2024harmbench}, reporting ASR@1. MoE scaling uses Qwen3-30B-A3B and Qwen3-235B-A22B.

\noindent\textbf{Baselines and implementation.}
We compare against FlowRL~\citep{zhu2026flowrl} as the distribution-matching baseline, and PPO~\citep{schulman2017proximal}, GRPO/GRPO+~\citep{shao2024deepseekmath, deepcoder2025}, and REINFORCE++~\citep{hu2025reinforcepp} as reward-maximization baselines. Our PyTorch implementation builds on \textit{verl}, \textit{deepcoder}, \textit{FlowRL}, and \textit{SEMA}, using their default training configurations unless otherwise noted. Detailed information such as dataset and framework links are provided in Appendix~\ref{app:setup}.
\vspace{-6pt}

\subsection{Diagnostic Investigation}
\label{subsec:diagnostic}
\begin{figure}[!h]
    \centering
    \vspace{-5pt}
    \begin{subfigure}[b]{0.48\textwidth}
        \centering
        \includegraphics[width=\textwidth]{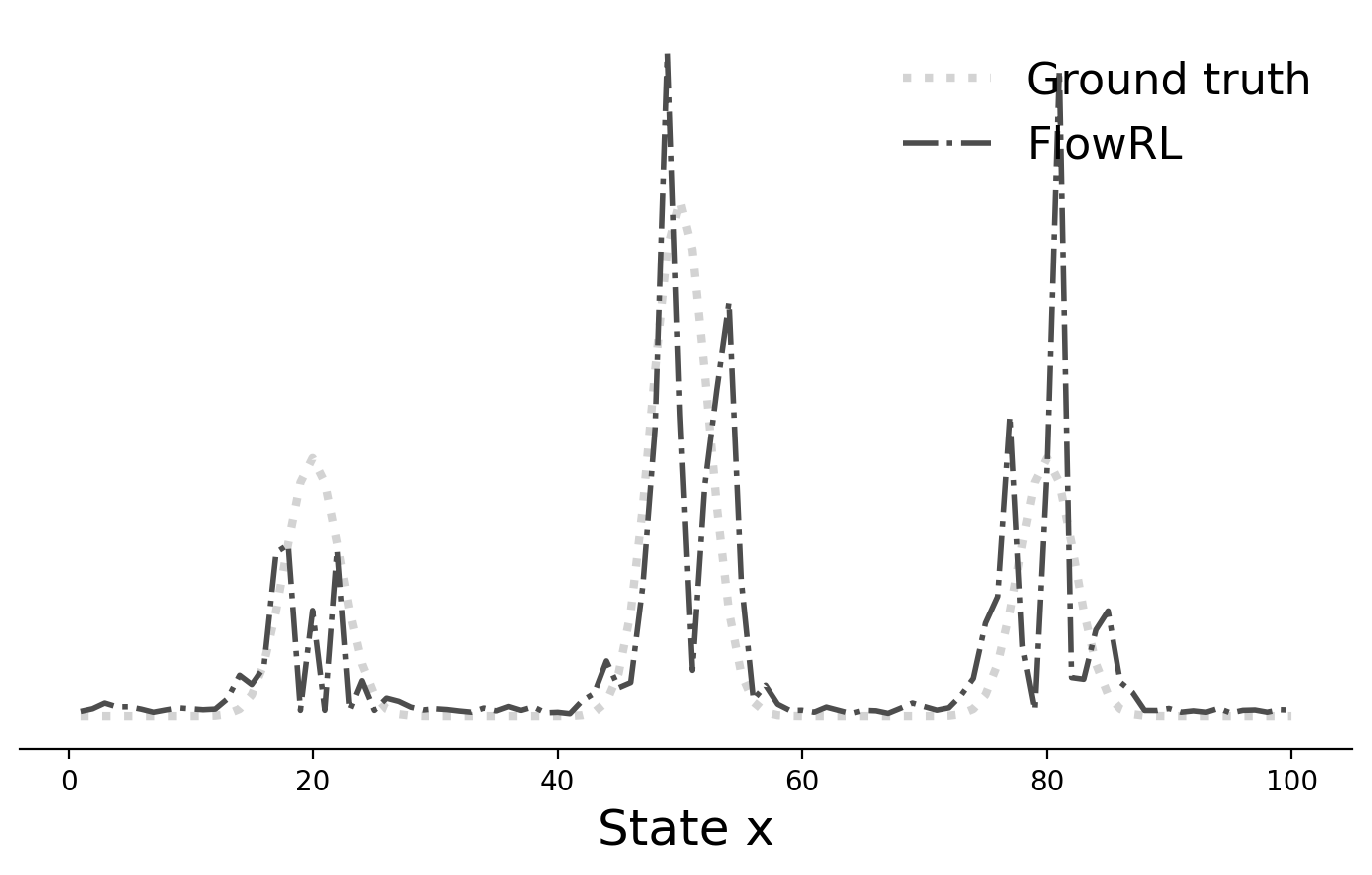}
        \caption{FlowRL}
        \label{fig:flowrl_dist500}
    \end{subfigure}
    \hfill
    \begin{subfigure}[b]{0.48\textwidth}
        \centering
        \includegraphics[width=\textwidth]{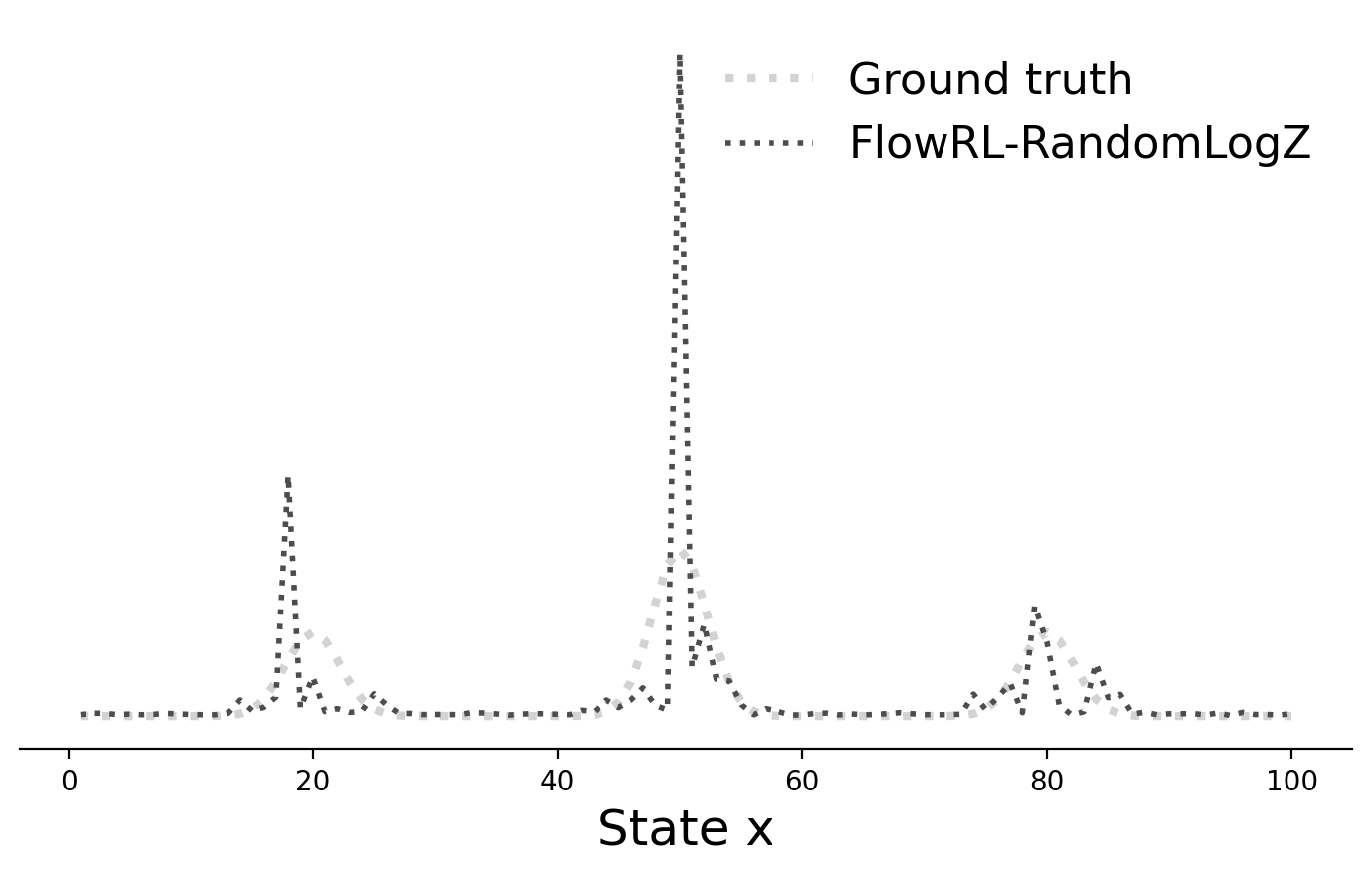}
        \caption{FlowRL-RandomLogZ}
        \label{fig:random_dist500}
    \end{subfigure}
    \vspace{0.2em} 
    \vspace{-3pt}
    \caption{Comparison of distribution matching across GRPO, FlowRL-RandomLogZ, FlowRL, and GFlowRL at the 500th training step.}
    \label{fig:dist500}
\end{figure}
\noindent\textbf{Random-$\mathbf{\log Z}$ replacement.}
We replace $\log Z_\phi(x)$ with an independent sample $\log Z \sim \mathcal{N}(0.5, 1)$ at each step, matching the empirical distribution of $\log Z_\phi(x)$ under standard FlowRL while keeping all other settings fixed. The results in Table~\ref{tab:math_7bresults} (FlowRL and FlowRL-RandomLogZ rows) are striking: replacing the learned $\log Z_\phi(x)$ with pure Gaussian noise \textit{does not} hurt average performance, and FlowRL-RandomLogZ even slightly exceeds standard FlowRL (36.19 vs.\ 35.63). This contradicts the hypothesis that $Z_\phi(x)$ functions as an informative prompt-conditional normalizer: if it did, randomizing it would cause a sharp drop. Instead, the learned partition function appears to contribute little beyond a noisy scalar bias, suggesting FlowRL's gains cannot be attributed to TB normalization and motivating the redesign in the next section.

To further probe this, we construct a synthetic target as a mixture of three Gaussians with distinct centers. As shown in Figures~\ref{fig:flowrl_dist500} and~\ref{fig:random_dist500}, FlowRL and FlowRL-RandomLogZ behave nearly identically, producing close-to-uniform samples that fail to capture the multimodal structure. This 100-dimensional setting reflects the few-hundred-step regime of LLM post-training, indicating the failure is not due to scale alone. Comparisons with GRPO and \method, along with training dynamics at steps $\{100, 500, 1000, 5000\}$, are provided in Appendix~\ref{app:dmc}.
\vspace{-5pt}
\subsection{Main Results}
\label{subsec:mainrst}
\vspace{-5pt}
\paragraph{Dense Models}

\begin{table*}[ht]
  \centering
  \caption{\textbf{Results on math reasoning benchmarks for dense models (Qwen2.5-7B).} We report Avg@16 accuracy with relative improvements over the backbone shown as subscripts. Positive gains are shown in \textcolor{scoregreen}{green} and negative changes in \textcolor{red}{red}. \method outperforms or matches all baselines.}
\label{tab:math_7bresults}
\vspace{-5pt}
  \resizebox{\textwidth}{!}{ 
  \begin{tabular}{l|llllll|l}
      \toprule[1.25pt]
        \textbf{Models} 
        & \textbf{AIME24} 
        & \textbf{AIME25}
        & \textbf{AMC23}
        & \textbf{MATH500}
        & \textbf{Minerva}
        & \textbf{Olympiad}
        & \textbf{Avg} \\
      \midrule[1.1pt]  
      \rowcolor[rgb]{0.93,0.93,0.93}
      \textcolor{gray}{\texttt{Backbone}} & \textcolor{gray}{4.38} & \textcolor{gray}{2.08} & \textcolor{gray}{30.78} & \textcolor{gray}{54.47} & \textcolor{gray}{22.38} & \textcolor{gray}{24.03} & \textcolor{gray}{23.02} \\
      R++ & $11.04_{\textcolor{scoregreen}{+6.66}}$ & $5.41_{\textcolor{scoregreen}{+3.33}}$ & $66.71_{\textcolor{scoregreen}{+35.93}}$ & $54.25_{\textcolor{red}{-0.22}}$ & $24.37_{\textcolor{scoregreen}{+1.99}}$ & $27.33_{\textcolor{scoregreen}{+3.30}}$ & 31.52 \\
      PPO & $9.38_{\textcolor{scoregreen}{+5.00}}$ & $7.29_{\textcolor{scoregreen}{+5.21}}$ & $63.43_{\textcolor{scoregreen}{+32.65}}$ & $57.98_{\textcolor{scoregreen}{+3.51}}$ & $26.53_{\textcolor{scoregreen}{+4.15}}$ & $27.25_{\textcolor{scoregreen}{+3.22}}$ & 31.98 \\
      GRPO & $13.54_{\textcolor{scoregreen}{+9.16}}$ & $9.79_{\textcolor{scoregreen}{+7.71}}$ & $64.53_{\textcolor{scoregreen}{+33.75}}$ & $57.05_{\textcolor{scoregreen}{+2.58}}$ & $23.06_{\textcolor{scoregreen}{+0.68}}$ & $26.88_{\textcolor{scoregreen}{+2.85}}$ & 32.48 \\
      FlowRL & $15.41_{\textcolor{scoregreen}{+11.03}}$ & $\mathbf{10.83}_{\textcolor{scoregreen}{+8.75}}$ & $54.53_{\textcolor{scoregreen}{+23.75}}$ & $66.96_{\textcolor{scoregreen}{+12.49}}$ & $31.41_{\textcolor{scoregreen}{+9.03}}$ & $34.61_{\textcolor{scoregreen}{+10.58}}$ & 35.63 \\
      FlowRL-RandomLogZ & $12.08_{\textcolor{scoregreen}{+7.70}}$ & $9.38_{\textcolor{scoregreen}{+7.30}}$ & $62.66_{\textcolor{scoregreen}{+31.88}}$ & $69.44_{\textcolor{scoregreen}{+14.97}}$ & $26.29_{\textcolor{scoregreen}{+3.91}}$ & $37.30_{\textcolor{scoregreen}{+13.27}}$ & 36.19 \\
      \midrule
      \method & $\mathbf{17.29}_{\textcolor{scoregreen}{+12.91}}$ & $9.79_{\textcolor{scoregreen}{+7.71}}$ & $\mathbf{67.66}_{\textcolor{scoregreen}{+36.88}}$ & $\mathbf{76.89}_{\textcolor{scoregreen}{+22.42}}$ & $\mathbf{33.62}_{\textcolor{scoregreen}{+11.24}}$ & $\mathbf{40.25}_{\textcolor{scoregreen}{+16.22}}$ & $\mathbf{40.92}$ \\
      \bottomrule[1.25pt]
  \end{tabular}}
  \vspace{-10pt}
\end{table*}
\begin{table*}[!h]
  \centering
  \caption{\textbf{Code benchmark results for dense models (DeepSeek-R1-Distill-Qwen-7B).}
Relative improvements are shown as subscripts, with gains in \textcolor{scoregreen}{green} and drops in \textcolor{red}{red}. \method achieves the best performance on all three benchmarks.
} 
\vspace{-5pt}
\label{tab:code_results}
  \resizebox{0.85\textwidth}{!}{  
  \begin{tabular}{l|ll|ll|l}
    \toprule[1.25pt]
    \textbf{Models} 
    & \multicolumn{2}{c|}{\textbf{LiveCodeBench}} 
    & \multicolumn{2}{c|}{\textbf{Codeforces}} 
    & \multicolumn{1}{c}{\textbf{HumanEval+}} \\  
    \cmidrule(lr){2-3} \cmidrule(lr){4-5} \cmidrule(lr){6-6}  
    & Avg@16 & Pass@16
    & Rating & Percentile
    & Avg@16 \\
    \midrule[1.1pt]
    \rowcolor[rgb]{0.93,0.93,0.93}
      \textcolor{gray}{Backbone}
    & \textcolor{gray}{30.68} & \textcolor{gray}{49.46}
    & \textcolor{gray}{886.68} & \textcolor{gray}{19.4\%}
    & \textcolor{gray}{80.90} \\
    
    R++ 
    & $30.46_{\textcolor{red}{-0.22}}$ & $52.68_{\textcolor{scoregreen}{+3.22}}$ 
    & $1208.03_{\textcolor{scoregreen}{+321.35}}$ & $56.8\%_{\textcolor{scoregreen}{+37.4\%}}$ 
    & $76.61_{\textcolor{red}{-4.29}}$ \\
    
    PPO               
    & $35.10_{\textcolor{scoregreen}{+4.42}}$ & $54.48_{\textcolor{scoregreen}{+5.02}}$
    & $1403.07_{\textcolor{scoregreen}{+516.39}}$ & $73.7\%_{\textcolor{scoregreen}{+54.3\%}}$
    & $82.32_{\textcolor{scoregreen}{+1.42}}$ \\
    
    GRPO              
    & $32.75_{\textcolor{scoregreen}{+2.07}}$ & $52.32_{\textcolor{scoregreen}{+2.86}}$
    & $1313.82_{\textcolor{scoregreen}{+427.14}}$ & $67.1\%_{\textcolor{scoregreen}{+47.7\%}}$
    & $80.13_{\textcolor{red}{-0.77}}$ \\
    FlowRL            
    & $37.43_{\textcolor{scoregreen}{+6.75}}$ & $56.27_{\textcolor{scoregreen}{+6.81}}$
    & $1549.47_{\textcolor{scoregreen}{+662.79}}$ & $83.3\%_{\textcolor{scoregreen}{+63.9\%}}$
    & $83.28_{\textcolor{scoregreen}{+2.38}}$ \\
    \midrule
    \method            
    & $\mathbf{38.62}_{\textcolor{scoregreen}{+7.94}}$ & $\mathbf{58.06}_{\textcolor{scoregreen}{+8.60}}$
    & $\mathbf{1646.21}_{\textcolor{scoregreen}{+759.73}}$ & $\mathbf{88.0\%}_{\textcolor{scoregreen}{+68.6\%}}$
    & $\mathbf{84.93}_{\textcolor{scoregreen}{+4.03}}$ \\
    \bottomrule[1.25pt]
  \end{tabular}}
  \vspace{-2pt}
\end{table*}
\begin{table*}[!h]
  \centering
  \vspace{-6pt}
  \caption{\textbf{ASR@1 $\uparrow$ on \textit{AdvBench} (LLM Classifier) and \textit{HarmBench} (HarmBench Classifier), dense models.} FlowRL fails to converge. Full results in Appendix~\ref{app:sema}.}
  \vspace{-5pt}
  \label{tab:sema_short}
  \resizebox{\textwidth}{!}{
  \begin{tabular}{l|llll|llll}
    \toprule[1.25pt]
    \textbf{Attackers / Victim Models}
    & \multicolumn{4}{c|}{\textbf{AdvBench} \citep{zou2023universal}}
    & \multicolumn{4}{c}{\textbf{HarmBench} \citep{mazeika2024harmbench}} \\
    \cmidrule(lr){2-5} \cmidrule(lr){6-9}
    & Qwen2.5-3B & Llama-3.1-8B & GPT-4.1-mini & Avg 
    & Qwen2.5-3B & Llama-3.1-8B & GPT-4.1-mini & Avg \\
    \midrule[1.1pt]
    X-Teaming \citep{rahman2025xteamingmultiturnjailbreaksdefenses}
    & 39.4 & 24.2 & 44.2 & 36.0
    & 45.3 & 22.0 & 44.7 & 37.3 \\
    SEMA \citep{feng2026sema}
    & 79.9 & 77.2 & 83.3 & 80.1
    & 74.5 & 70.6 & 79.8 & 75.0 \\
    \midrule
    \method & \textbf{80.2} & \textbf{81.2} & \textbf{86.1} & \textbf{82.5} & \textbf{79.9} & \textbf{73.0} & \textbf{85.5} & \textbf{79.5} \\
    \bottomrule[1.25pt]
  \end{tabular}}
  \vspace{-20pt}
\end{table*}

\begin{figure}[h]
    \centering
    \includegraphics[width=0.98\linewidth]{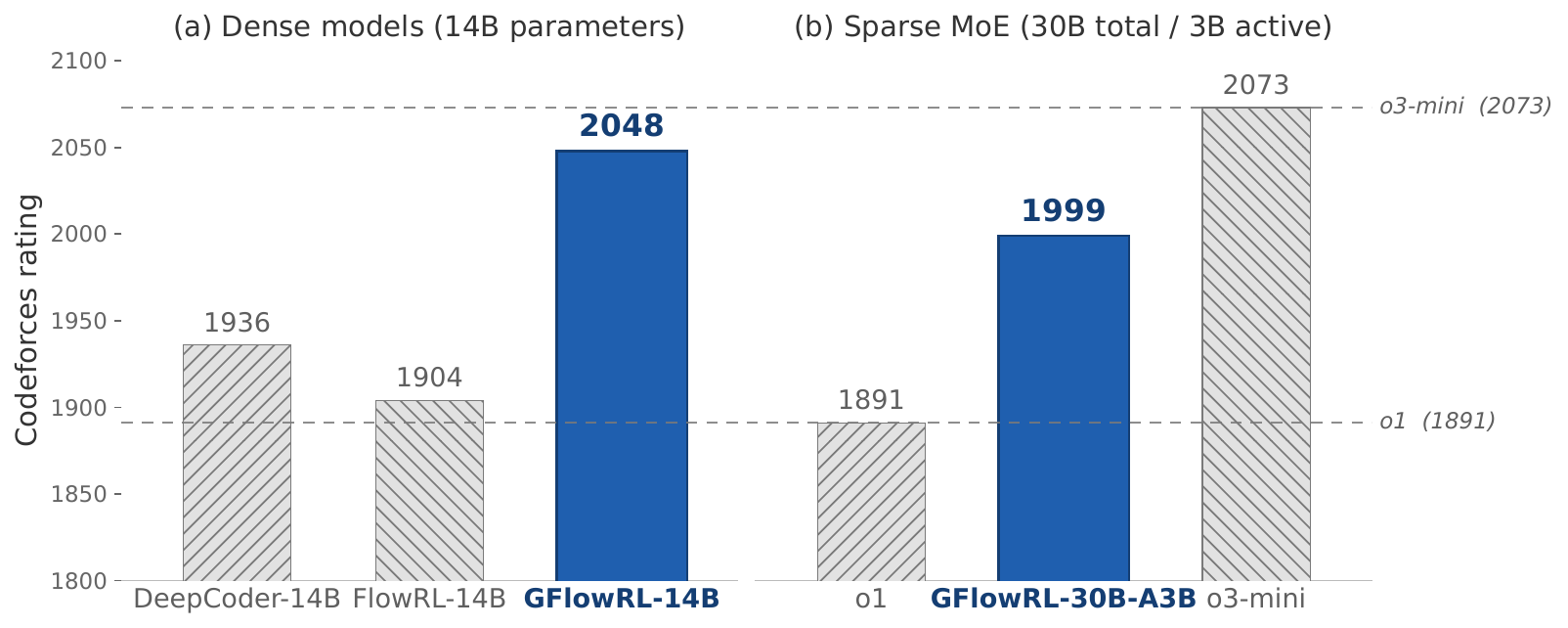}
    \vspace{-5pt}
    \caption{Codeforces rating comparison across Dense and Sparse MoE model families.}
    \label{fig:cf-comparison}
\vspace{-2pt}
\end{figure}
We evaluate \method on dense backbones across mathematical reasoning (Table~\ref{tab:math_7bresults}), competitive programming and code generation (Table~\ref{tab:code_results}, Figure~\ref{fig:cf-comparison}), and adversarial red-teaming (Table~\ref{tab:sema_short}), comparing against reward-maximizing baselines (R++, PPO, GRPO) and FlowRL. The diagnostic predicts two things: replacing $Z_\phi$ with the in-batch estimate should not hurt accuracy where FlowRL converges, and should restore convergence where it does not. Results at 32B follow the same trends and are reported in Appendix~\ref{app:math_32b}.

\method and FlowRL substantially improve over PPO and GRPO on both math and code. On 7B math, \method exceeds GRPO by $+8.44$ average points (40.92 vs.\ 32.48); on Codeforces, both GFlowNet-style methods open a 200--330 Elo advantage over GRPO. On every benchmark where both converge, \method also exceeds FlowRL: best 7B math average (\textit{40.92}, $+5.29$), highest LiveCodeBench Avg@16 (\textit{38.62}, $+1.19$), highest Codeforces rating (\textit{1646}, $+97$) and percentile (\textit{88.0\%}, $+4.7\%$), and highest HumanEval+ score (\textit{84.93}, $+1.65$). Combined with the diagnostic in Sec.~\ref{subsec:diagnostic}, this confirms that \method delivers the benefits of TB with an efficient and stable estimator.

\textit{\method succeeds where FlowRL fails to converge.}
The strongest evidence for the streamlined design comes from adversarial red-teaming, where the reward signal is noisier and sparser than in math or code. On AdvBench and HarmBench (Table~\ref{tab:sema_short}), FlowRL fails to converge and produces no usable attacker, while \method trains stably and attains the highest average ASR@1 on both benchmarks (\textit{82.5\%} on AdvBench, \textit{79.5\%} on HarmBench), surpassing the SOTA multi-turn attacker SEMA~\citep{feng2026sema} by $+2.4$ and $+4.5$ points respectively. Removing the high-variance $\log Z_\phi$ term restores stable optimization, 
turning parity on math/code into a strict capability advantage: \method covers the union of tasks where either GFlowNet-style or GRPO/PPO training works.

\textit{Comparison with frontier reasoning models.}
Figure~\ref{fig:cf-comparison}a compares \method's 14B Codeforces rating against open-source dense baselines and proprietary reasoning systems. \method reaches \textit{2048} Elo, exceeding DeepCoder-14B by $+112$ and FlowRL-14B by $+144$ to set a new open-source state of the art at this scale; DeepCoder, a strong GRPO recipe with targeted improvements that previously held this bar, makes the margin notable. \method also outperforms OpenAI's o1 by $+157$ (1891) and closes to within $25$ of o3-mini (2073)\footnote{\url{https://openai.com/index/introducing-o3-and-o4-mini/}}, a regime previously reached only by frontier closed-source systems. Detailed settings are reported in Appendix~\ref{tab:training-configs-deepcoder}.
\begin{table*}[!h]
  \centering
  \vspace{-2pt}
  \caption{\textbf{Results on math reasoning benchmarks for sparse models.} We report Avg@16 accuracy with relative improvements shown as subscripts. Positive gains are shown in \textcolor{scoregreen}{green} and negative changes in \textcolor{red}{red}. \method outperforms or matches for all baselines across 30B-A3B, and 235B-A22B  model scales. FlowRL fails to converge in both cases.}
\label{tab:math_sparse_results}
\vspace{-5pt}
  \resizebox{\textwidth}{!}{ 
  \begin{tabular}{l|llllll|l}
      \toprule[1.25pt]
        \textbf{Models} 
        & \textbf{AIME24} 
        & \textbf{AIME25}
        & \textbf{AMC23}
        & \textbf{MATH500}
        & \textbf{Minerva}
        & \textbf{Olympiad}
        & \textbf{Avg} \\
      \midrule[1.1pt]
      \midrule
      \rowcolor[rgb]{0.93,0.93,0.93}
      \multicolumn{8}{c}{\texttt{Qwen3-30B-A3B, Max Response Len = 32K tokens}} \\
      \textcolor{gray}{\texttt{Backbone}} & \textcolor{gray}{80.83} & \textcolor{gray}{68.33} & \textcolor{gray}{92.50} & \textcolor{gray}{88.50} & \textcolor{gray}{43.41} & \textcolor{gray}{73.57} & \textcolor{gray}{74.52} \\
      GRPO & $80.20_{\textcolor{red}{-0.63}}$ & $71.00_{\textcolor{scoregreen}{+2.67}}$ & $95.63_{\textcolor{scoregreen}{+3.13}}$ & $89.18_{\textcolor{scoregreen}{+0.68}}$ & $42.53_{\textcolor{red}{-0.88}}$ & $76.15_{\textcolor{scoregreen}{+2.58}}$ & 75.78 \\
      \midrule
      \method & $\textbf{81.88}_{\textcolor{scoregreen}{+1.05}}$ & $\textbf{71.04}_{\textcolor{scoregreen}{+2.71}}$ & $\textbf{97.97}_{\textcolor{scoregreen}{+5.47}}$ & $\textbf{93.29}_{\textcolor{scoregreen}{+4.79}}$ & $\textbf{45.59}_{\textcolor{scoregreen}{+2.18}}$ & $\textbf{80.14}_{\textcolor{scoregreen}{+6.57}}$ & \textbf{78.32} \\
      
      \midrule
      \multicolumn{8}{c}{\texttt{Qwen3-235B-A22B, Max Response Len = 32K tokens}} \\
      \textcolor{gray}{\texttt{Backbone}} & \textcolor{gray}{84.58} & \textcolor{gray}{79.17} & \textcolor{gray}{95.78} & \textcolor{gray}{95.67} & \textcolor{gray}{49.86} & \textcolor{gray}{82.70} & \textcolor{gray}{81.29} \\  
      GRPO & $84.80_{\textcolor{scoregreen}{+0.22}}$ & $80.20_{\textcolor{scoregreen}{+1.03}}$ & $98.75_{\textcolor{scoregreen}{+2.97}}$ & $96.33_{\textcolor{scoregreen}{+0.66}}$ & $50.48_{\textcolor{scoregreen}{+0.62}}$ & $83.84_{\textcolor{scoregreen}{+1.14}}$ & 82.40 \\
      \method & $\textbf{85.21}_{\textcolor{scoregreen}{+0.63}}$ & $\textbf{80.63}_{\textcolor{scoregreen}{+1.46}}$ & $\textbf{99.69}_{\textcolor{scoregreen}{+3.91}}$ & $\textbf{98.36}_{\textcolor{scoregreen}{+2.69}}$ & $\textbf{51.13}_{\textcolor{scoregreen}{+1.27}}$ & $\textbf{85.06}_{\textcolor{scoregreen}{+2.36}}$ & \textbf{83.35} \\    
    \bottomrule[1.25pt]
  \end{tabular}}
  \vspace{-10pt}
\end{table*}

\noindent\textbf{MoE Models}
MoE backbones amplify the instability identified in Sec.~\ref{subsec:diagnostic}: nondeterministic routing inflates gradient variance, which combined with the high-variance $\log Z_\phi$ term breaks optimization. To our knowledge, no prior GFlowNet-style RL method has trained stably on MoE at reasoning scale.
We evaluate \method on Qwen3-30B-A3B across code (Figure~\ref{fig:cf-comparison}b) and math (Table~\ref{tab:math_sparse_results}), and further scale to Qwen3-235B-A22B on math as a frontier-capacity stress test.

\textit{\method trains stably on MoE.} On code, \method-30B-A3B (from Qwen3-30B-A3B) reaches \textit{1999} Elo with only 3B active parameters, exceeding OpenAI's o1 by $+108$ points and closing to within $74$ of o3-mini. On math, \method improves over the backbone by $+3.80$ average points (\textit{78.32} vs.\ 74.52) and over GRPO by $+2.54$ (vs.\ 75.78), with gains on every benchmark. FlowRL diverges across all benchmarks on both backbones and produces no usable checkpoints, which we attribute to the instability of its auxiliary partition-function network under sparse routing; \method's streamlined estimator sidesteps this entirely.

\textit{Scaling MoE.} A central question for any RL recipe is whether it survives frontier scale, where instabilities tolerable at 30B become catastrophic and each failed run carries substantial compute cost. We apply \method to Qwen3-235B-A22B without modification, reusing the 30B hyperparameters, data, and schedule. Despite the $7.8\times$ increase in parameters, training remains stable end-to-end: \method improves over the backbone by $+2.06$ average points (\textit{83.35} vs.\ 81.29) and over GRPO by $+0.95$ (vs.\ 82.40), with gains on every benchmark. 
Notably, due to GPU constraints \method is trained for only 30 steps vs. 100 for GRPO, achieving this margin with roughly $3\times$ fewer steps; training had not yet plateaued when stopped. FlowRL again fails to converge. To our knowledge, this is the first GFlowNet-style RL method trained stably on a 235B MoE, and the fact that the recipe transfers suggests the streamlined estimator removes a fundamental scaling bottleneck.
\vspace{-6pt}
\subsection{Ablation Study}
\label{subsec:astudy}
\vspace{-6pt}
We ablate the key design choices in \method to identify the source of its gains. All ablations use Qwen2.5-7B on math benchmarks, trained for 100 steps.
\begin{table*}[!h]
\centering
\vspace{-8pt}
\caption{\textbf{Ablation studies.} Left: effect of $\beta$. Right: diversity scores across RL algorithms.}
\label{tab:ablation_combined}
\vspace{-5pt}
\begin{subtable}[t]{0.48\textwidth}
\centering
\caption{Effect of $\beta$ on math reasoning.}
\label{tab:beta}
\small
\begin{tabular}{l|ccccc}
\toprule[1.25pt]
\textbf{$\beta$} & 1 & 5 & 8 & 10 & 15 \\
\midrule[1.1pt]
\method & 37.9 & 37.2 & \textbf{38.2} & 37.7 & 36.4 \\
\midrule
\end{tabular}
\end{subtable}
\hfill
\begin{subtable}[t]{0.48\textwidth}
\centering
\caption{Diversity score across RL algorithms.}
\label{tab:diversity}
\small
\setlength{\tabcolsep}{4pt}
\begin{tabular}{l|cccc}
\toprule[1.25pt]
\textbf{Models} & GRPO & PPO & FlowRL & \method \\
\midrule[1.1pt]
Diversity & 1.21 & 1.15 & 2.64 & \textbf{3.93} \\
\midrule
\end{tabular}
\end{subtable}

\vspace{-12pt}
\end{table*}

\noindent\textbf{Effect of $\beta$ in Eq~\ref{eq:batch_logZ}.}
The inverse temperature $\beta$ controls the sharpness of the target distribution. Table~\ref{tab:beta} reports average accuracy across six math benchmarks for $\beta \in \{1, 5, 8, 10, 15\}$: accuracy follows an inverted-U trend peaking at $\mathbf{38.2}$ at $\beta = 8$, and stays within 0.5 points of this optimum for all $\beta \in [1, 10]$, indicating low sensitivity. We use $\beta = 8$ by default.

\begin{table*}[!h]
  \centering
  \small
\vspace{-6pt}
  \caption{\textbf{Effect of flow-gap clipping.} We compare the full \method against an unclipped variant.}
  \label{tab:flow_gap}
  \vspace{-5pt}
  \begin{subtable}[t]{0.4\textwidth}
    \centering
    \caption{Average accuracy on 6 math benchmarks.}
    \label{tab:flow_gap_acc}
    {
    \begin{tabular}{l|l}
      \toprule[1.25pt]
      \textbf{Models} & \textbf{Average} \\
      \midrule[1.1pt]
      \method         &  40.92\\
      \method-no-clip &  37.02\\
      \bottomrule[1.25pt]
    \end{tabular}}
  \end{subtable}
  \hfill
  \begin{subtable}[t]{0.58\textwidth}
    \centering
    \caption{Gradient norm statistics.}
    \label{tab:flow_gap_grad}
    {
    \begin{tabular}{l|lll}
      \toprule[1.25pt]
      \textbf{Models} & \textbf{Mean} & \textbf{Max} & \textbf{Std} \\
      \midrule[1.1pt]
      \method         &0.095  & 6.184 &  0.4301 \\
      \method-no-clip &0.601 & 16.57 &  1.3785\\
      \bottomrule[1.25pt]
    \end{tabular}}
  \end{subtable}
  \vspace{-12pt}
\end{table*}

\noindent\textbf{Flow-gap clipping.} We ablate flow-gap clipping by comparing \method against an unclipped variant. Removing clipping drops average accuracy by 3.9 points (Table~\ref{tab:flow_gap_acc}) and inflates gradient norms substantially, with the mean rising $6.3\times$ and the maximum nearly tripling (Table~\ref{tab:flow_gap_grad}). These unstable gradients impair optimization, confirming that clipping is essential for stable training.

\noindent\textbf{Diversity analysis.}
Table~\ref{tab:diversity} reports diversity scores from GPT-o4-mini on a 1--5 scale following~\citep{yu2025flow, zhu2026flowrl}, averaged over five runs (variance $\approx 0.53$; prompt in Figure~\ref{fig:diversity-prompt} in Appendix~\ref{app:divs}). GRPO ($1.21$) and PPO ($1.15$) fall in the ``minimal diversity'' band, consistent with reward-maximizing policy gradients. FlowRL more than doubles this to $2.64$ by matching the reward distribution. GFlowRL further improves to $3.93$ (``high diversity''), $3.2\times$ GRPO and $1.6\times$ FlowRL; the $2.7$-point gap over PPO/GRPO is $5\times$ judge variance, well outside measurement noise.

\vspace{-5pt}
\section{Related Work}
\label{sec:related}
\vspace{-5pt}
\noindent\textbf{RL for LLM reasoning.} RLHF established post-training with PPO-style optimization and a reference-policy regularizer~\citep{ouyang2022training,schulman2017proximal}. Recent reasoning systems shift to verifiable rewards on math and code, enabling critic-free or group-based optimization at scale (GRPO, REINFORCE++, DAPO, DeepSeek-R1, Kimi k1.5, DeepCoder)~\citep{shao2024deepseekmath,hu2025reinforcepp,yu2025dapo,guo2025deepseek,team2025kimi,deepcoder2025}. These methods maximize expected reward and tend to concentrate mass on a small set of dominant trajectories.

\noindent\textbf{GFlowNets and distribution matching.} GFlowNets sample compositional objects in proportion to reward rather than greedily maximizing it~\citep{bengio2021flow,bengio2023gflownet,malkin2022tb}. The closest prior work for LLM post-training is FOR and FlowRL~\citep{yu2025flow,zhu2026flowrl}, both of which instantiate the trajectory balance objective with a learned prompt-conditional partition network. \method stays in this distribution-matching line but replaces that auxiliary network with a rollout-based estimate, motivated by its instability when jointly trained with a frontier-scale policy.

\noindent\textbf{In-batch estimation and variance reduction.} Our estimator connects to a long line of variance-reduced score-function estimators, from REINFORCE baselines~\citep{williams1992simple} to leave-one-out and control-variate methods~\citep{richter2020vargrad,Grathwohl2018BackpropagationTT}, and in LLM RL to group-relative normalization in GRPO and REINFORCE++~\citep{shao2024deepseekmath,hu2025reinforcepp}. We use the rollout group to estimate the partition term implied by the distribution-matching objective and inject it as a stop-gradient baseline, structurally analogous to a GRPO group-mean baseline but applied to the trajectory-balance residual rather than the advantage. Extended discussion in Appendix~\ref{app:related}.
\vspace{-6pt}
\section{Conclusion}
\label{sec:conclusion}
\vspace{-6pt}
Distribution-matching RL is appealing for LLM reasoning because it preserves multiple high-reward solution paths, but our results suggest the standard GFlowNet recipe does not transfer cleanly to modern post-training. The bottleneck is not the trajectory-balance objective itself but the learned partition function: in the short-horizon regime of LLM RL, a randomly initialized $Z_\phi$ must be learned from scratch alongside a large pretrained policy that only needs modest refinement, turning the partition term into a source of noise, gradient instability, and systems overhead. We proposed \method, which replaces the learned partition network with an in-batch Monte Carlo estimate from the rollout group already required for training, combined with importance sampling and flow-gap clipping for stability. The result is a simpler algorithm with a GRPO-compatible pipeline that retains the reward-distribution-matching perspective: it matches or improves on FlowRL where FlowRL converges, restores convergence in noisier settings such as red-teaming, and scales cleanly to large MoE backbones. The broader takeaway is that scaling GFlowNet-style RL to large reasoning models may depend less on adding auxiliary machinery and more on identifying which parts of the original objective are unnecessary in this regime.
\section{Acknowledgment}
\label{sec:ack}
We thank Ed Nightingale for sharing GPU resources during the early stages of this work; Lili Chen for her support and valuable feedback; GCR and Lifeng Li for their assistance with the GPU cluster; Hao Tang for his insightful comments on the paper; and Mingqian Feng for helping with the SEMA experiments.
\bibliography{main}
\bibliographystyle{plainnat}

\appendix
\newpage
\section{Limitations and Broader Impacts}
\label{app:limits}
\method's in-batch Monte Carlo estimate of $\log Z$ can have higher variance in principle, especially when the group size is small. In practice, our two stabilizers, importance-sampling correction and flow-gap clipping, help control this variance, and the resulting streamlined estimator consistently outperforms the learned alternative across our evaluated settings. Nevertheless, a more precise characterization of the bias-variance tradeoff, as well as lower-variance group-based estimators, remains an important direction for future work. Another limitation is generality. Our experiments cover math reasoning, code reasoning, and red-teaming for language models, but it remains unclear whether the same estimator-centric design principle extends to broader agentic or multimodal RL settings. We leave this broader validation to future work.

\section{Theoretical Justification for \method}
\label{app:proof}
We provide formal proofs of the two theoretical claims made in Section~\ref{sec:method} that the \method loss recovers the reward-proportional target distribution at its fixed point (Proposition~\ref{prop:optimality}). The first result is stated for the loss without length normalization and in a neighborhood where the flow-gap clip is inactive; Remark~\ref{rem:length_norm} discusses the effect of length normalization, and Remark~\ref{rem:clip} discusses the clip.

\begin{proposition}[Fixed-point optimality of \method]
\label{prop:optimality}
Consider the population form of the \method loss in Eq.~\eqref{eq:gflowrl_loss} \emph{without} length normalization and with the flow-gap clip inactive (i.e.\ $\tilde{g}=g$),
\begin{equation}
\mathcal{L}_{\text{\method}}(\theta;\mathbf{x})
\;=\; \mathbb{E}_{\mathbf{y}\sim\pi_{\phi_{\text{old}}}(\cdot\mid\mathbf{x})}\!\left[
w(\mathbf{y})\!\left(\operatorname{sg}\!\left[\mathcal{Z}_t(\mathbf{x})\right]
\;+\; \log\frac{\pi_\theta(\mathbf{y}\mid\mathbf{x})}{\pi_{\theta_{\text{ref}}}(\mathbf{y}\mid\mathbf{x})}
\;-\; {\beta}{r(\mathbf{x},\mathbf{y})}\right)^{\!2}\right],
\label{eq:pop_loss}
\end{equation}


where $w(\mathbf{y}) = \min\!\big(\pi_\theta(\mathbf{y}\mid\mathbf{x})/\pi_{\phi_{\text{old}}}(\mathbf{y}\mid\mathbf{x}),\,1+\epsilon\big)>0$ and
\begin{equation}
\mathcal{Z}_t(\mathbf{x}) \;=\; \mathbb{E}_{\mathbf{y}\sim\pi_{\phi_{\text{old}}}(\cdot\mid\mathbf{x})}\!\left[
{\beta}{r(\mathbf{x},\mathbf{y})} \;+\; \log\pi_{\theta_{\text{ref}}}(\mathbf{y}\mid\mathbf{x}) \;-\; \log\pi_{\phi_{\text{old}}}(\mathbf{y}\mid\mathbf{x})\right]
\label{eq:pop_estimator}
\end{equation}
is the population analogue of the in-batch estimator in Eq.~\eqref{eq:batch_logZ}. Assume the policy class is expressive enough to contain the target $p^*(\mathbf{y}\mid\mathbf{x}) \propto \pi_{\theta_{\text{ref}}}(\mathbf{y}\mid\mathbf{x})\exp(\beta r(\mathbf{x},\mathbf{y}))$. Then at any self-consistent fixed point $\theta^*$ with $\pi_{\theta^*} = \pi_{\phi_{\text{old}}}$ that achieves $\mathcal{L}_{\text{\method}}(\theta^*;\mathbf{x}) = 0$,
\begin{equation*}
\pi_{\theta^*}(\mathbf{y}\mid\mathbf{x}) \;=\; \frac{\pi_{\theta_{\text{ref}}}(\mathbf{y}\mid\mathbf{x})\exp\!\big(\beta r(\mathbf{x},\mathbf{y})\big)}{Z(\mathbf{x})} \;=\; p^*(\mathbf{y}\mid\mathbf{x}),\;
Z(\mathbf{x}) \;=\; \sum_{\mathbf{y}'}\pi_{\theta_{\text{ref}}}(\mathbf{y}'\mid\mathbf{x})\exp\!\big(\beta r(\mathbf{x},\mathbf{y}')\big),
\end{equation*}
and the estimator $\mathcal{Z}_t(\mathbf{x})$ coincides with the true $\log Z(\mathbf{x})$.
\end{proposition}

\begin{proof}
The argument has three steps: (i) reduce the residual at a fixed point, (ii) derive the form of $\pi_{\theta^*}$, and (iii) verify self-consistency of the estimator.

\medskip\noindent\textbf{Step 1: Residual at a fixed point.}\quad
At a fixed point where $\pi_{\theta^*}=\pi_{\phi_{\text{old}}}$, the log-ratio term $\log(\pi_\theta/\pi_{\theta_{\text{ref}}})$ in Eq.~\eqref{eq:pop_loss} evaluates to $\log(\pi_{\theta^*}/\pi_{\theta_{\text{ref}}})$, so the residual becomes
\begin{equation}
R(\mathbf{y}) \;\coloneqq\; \mathcal{Z}_t(\mathbf{x}) \;+\; \log\frac{\pi_{\theta^*}(\mathbf{y}\mid\mathbf{x})}{\pi_{\theta_{\text{ref}}}(\mathbf{y}\mid\mathbf{x})} \;-\; \beta {r(\mathbf{x},\mathbf{y})}.
\label{eq:residual}
\end{equation}
Because $w(\mathbf{y})>0$, Eq.~\eqref{eq:pop_loss} is a weighted second moment of $R$ under $\pi_{\phi_{\text{old}}}$. It equals zero iff $R(\mathbf{y})=0$ holds $\pi_{\phi_{\text{old}}}$-almost-surely, i.e.\ for every $\mathbf{y}$ in the support of $\pi_{\theta^*}$.

\medskip\noindent\textbf{Step 2: Form of $\pi_{\theta^*}$.}\quad
Setting $R(\mathbf{y})=0$ and rearranging,
\begin{equation}
\log\pi_{\theta^*}(\mathbf{y}\mid\mathbf{x}) \;=\; \log\pi_{\theta_{\text{ref}}}(\mathbf{y}\mid\mathbf{x}) \;+\; {\beta} {r(\mathbf{x},\mathbf{y})} \;-\; \mathcal{Z}_t(\mathbf{x}).
\label{eq:log_policy}
\end{equation}
 Note that $\mathcal{Z}_t(\mathbf{x})$ is a scalar depending only on $\mathbf{x}$ (the stop-gradient operator treats it as a constant w.r.t.\ $\theta$). Exponentiating,
\begin{equation}
\pi_{\theta^*}(\mathbf{y}\mid\mathbf{x}) \;=\; \pi_{\theta_{\text{ref}}}(\mathbf{y}\mid\mathbf{x})\exp\!\big(\beta r(\mathbf{x},\mathbf{y})\big)\cdot e^{-\mathcal{Z}_t(\mathbf{x})}.
\label{eq:pi_form}
\end{equation}
Imposing $\sum_{\mathbf{y}}\pi_{\theta^*}(\mathbf{y}\mid\mathbf{x})=1$ uniquely determines $\mathcal{Z}_t(\mathbf{x})$:
\begin{equation}
e^{\mathcal{Z}_t(\mathbf{x})} \;=\; \sum_{\mathbf{y}}\pi_{\theta_{\text{ref}}}(\mathbf{y}\mid\mathbf{x})\exp\!\big(\beta r(\mathbf{x},\mathbf{y})\big) \;=\; Z(\mathbf{x}),
\qquad\text{hence}\qquad\mathcal{Z}_t(\mathbf{x}) \;=\; \log Z(\mathbf{x}).
\label{eq:C_eq_logZ}
\end{equation}
Substituting back into Eq.~\eqref{eq:pi_form} gives $\pi_{\theta^*}(\mathbf{y}\mid\mathbf{x}) = \pi_{\theta_{\text{ref}}}(\mathbf{y}\mid\mathbf{x})\exp(\beta r(\mathbf{x},\mathbf{y}))/Z(\mathbf{x}) = p^*(\mathbf{y}\mid\mathbf{x})$.

\medskip\noindent\textbf{Step 3: Self-consistency of the estimator.}\quad
Step~2 forced $\log Z(\mathbf{x})$ by the normalization of $\pi_{\theta^*}$. It remains to verify that the \emph{definition} of $\mathcal{Z}_t(\mathbf{x})$ in Eq.~\eqref{eq:pop_estimator} is consistent with this value, i.e.\ that the self-consistent fixed point exists. Plugging $\pi_{\phi_{\text{old}}} = \pi_{\theta^*}$ and Eq.~\eqref{eq:log_policy}into Eq.~\eqref{eq:pop_estimator}:
\begin{align*}
\mathcal{Z}_t(\mathbf{x})
&\;=\; \mathbb{E}_{\mathbf{y}\sim\pi_{\theta^*}}\!\left[\beta {r(\mathbf{x},\mathbf{y})} + \log\pi_{\theta_{\text{ref}}}(\mathbf{y}\mid\mathbf{x}) - \log\pi_{\theta^*}(\mathbf{y}\mid\mathbf{x})\right] \\
&\;=\; \mathbb{E}_{\mathbf{y}\sim\pi_{\theta^*}}\!\left[\beta {r(\mathbf{x},\mathbf{y})} + \log\pi_{\theta_{\text{ref}}}(\mathbf{y}\mid\mathbf{x}) - \Big(\log\pi_{\theta_{\text{ref}}}(\mathbf{y}\mid\mathbf{x}) + {\beta}{r(\mathbf{x},\mathbf{y})} - \log Z(\mathbf{x})\Big)\right] \\
&\;=\; \log Z(\mathbf{x}),
\end{align*}
matching Eq.~\eqref{eq:C_eq_logZ}. Thus the pair $(\pi_{\theta^*},\mathcal{Z}_t(\mathbf{x}))=(p^*,\log Z)$ is a self-consistent fixed point of the joint system (policy update, estimator update), completing the proof. \qed
\end{proof}

\begin{remark}[Finite-batch estimator]
\label{rem:finite_batch}
The in-batch estimator in Eq.~\eqref{eq:batch_logZ} is an unbiased Monte Carlo estimate of Eq.~\eqref{eq:pop_estimator} with variance $O(1/G)$.
\end{remark}

\begin{remark}[Inactivity of the flow-gap clip at the fixed point]
\label{rem:clip}
Proposition~\ref{prop:optimality} is stated for the unclipped loss, but the argument extends to the clipped loss in Eq.~\eqref{eq:gflowrl_loss}. At the fixed point $\pi_{\theta^*}=\pi_{\phi_{\text{old}}}$, the flow gap defined in Eq.~\eqref{eq:flow_gap} reduces to the residual in Eq.~\eqref{eq:residual} (without length normalization), which equals zero by the fixed-point condition. Hence $g^{(i)}=0$ for every rollout at the optimum, the asymmetric clip $[-\epsilon_{\text{low}}, +\epsilon_{\text{high}}]$ is inactive, and $\tilde{g}^{(i)}=g^{(i)}$. Proposition~\ref{prop:optimality} therefore characterizes a local fixed point of the clipped loss as well: the clip acts as a trust-region safeguard on outlier gradients but does not shift the stationary distribution. Away from the fixed point the clip may bias individual updates, but this is the intended effect — it bounds the influence of outlier rollouts with $|g^{(i)}|$ larger than the trust-region radius.
\end{remark}

\begin{remark}[Effect of length normalization]
\label{rem:length_norm}
The \method loss in Eq.~\eqref{eq:gflowrl_loss} normalizes $\log(\pi_\theta/\pi_{\phi_{\text{old}}})$ by the response length $|\mathbf{y}|$, and the corresponding flow-gap residual (Eqs.~\ref{eq:residual_main}--\ref{eq:flow_gap}) normalizes $\log(\pi_{\phi_{\text{old}}}/\pi_{\theta_{\text{ref}}})$ likewise. Repeating the fixed-point argument of Step~1, with the clip inactive, the residual at $\pi_{\theta^*}=\pi_{\phi_{\text{old}}}$ becomes
\begin{equation*}
\mathcal{Z}_t(\mathbf{x}) \;+\; \frac{1}{|\mathbf{y}|}\log\frac{\pi_{\theta^*}(\mathbf{y}\mid\mathbf{x})}{\pi_{\theta_{\text{ref}}}(\mathbf{y}\mid\mathbf{x})} \;-\; \beta {r(\mathbf{x},\mathbf{y})} \;=\; 0,
\end{equation*}
so
\begin{equation*}
\pi_{\theta^*}(\mathbf{y}\mid\mathbf{x}) \;=\; \pi_{\theta_{\text{ref}}}(\mathbf{y}\mid\mathbf{x})\,\exp\!\Big(|\mathbf{y}|\cdot\big(\beta r(\mathbf{x},\mathbf{y}) - \mathcal{Z}_t(\mathbf{x})\big)\Big).
\end{equation*}
This induces a length-dependent variant of $p^*$. At a zero-residual fixed point, a sequence of length $L$ is weighted relative to the reference by $\exp\!\big(L\cdot \beta r\big)$, together with the length-dependent normalization factor $\exp(-L\mathcal{Z}_t(\mathbf{x}))$. Thus, in the idealized setting where all rollouts share a common length $|\mathbf{y}^{(i)}|=L$, the fixed point recovers the reward-tilted distribution at inverse temperature $L\beta$; equivalently, Proposition~\ref{prop:optimality} applies with $\beta\mapsto L\beta$. When response lengths vary, the term $\exp(-|\mathbf{y}|\mathcal{Z}_t(\mathbf{x}))$ is no longer a global normalization constant across sequences, so length normalization introduces a length-dependent distortion of the stationary distribution. In practice, response lengths vary only modestly within rollout groups on reasoning benchmarks, so this distortion is limited. We therefore view length normalization as an engineering trade-off: it improves optimization stability by preventing long sequences from dominating the loss, at the cost of a mild length-dependent bias in the induced fixed point.
\end{remark}

\section{Extended Related Work}
\label{app:related}

\noindent\textbf{RL for LLM post-training and reasoning.}
RLHF established sequence-level post-training for aligning language models, typically using PPO-style optimization with a reference-policy regularizer~\citep{ouyang2022training,schulman2017proximal}. More recent reasoning systems shift to verifiable rewards on math and code, which makes critic-free or group-based optimization practical at scale. Representative examples include GRPO, REINFORCE++, DAPO, DeepSeek-R1, Kimi k1.5, and DeepCoder~\citep{shao2024deepseekmath,hu2025reinforcepp,yu2025dapo,guo2025deepseek,team2025kimi,deepcoder2025}. Many of these systems are built around reward-maximizing or normalized-advantage objectives and have produced strong reasoning performance, but they do not explicitly target a reward-weighted response distribution and can therefore favor a small set of dominant trajectories.

\noindent\textbf{GFlowNets and distribution matching.}
GFlowNets were introduced as a framework for sampling compositional objects in proportion to reward rather than greedily maximizing it~\citep{bengio2021flow,bengio2023gflownet,zhang2022generative}. Subsequent work developed stronger objectives and interpretations, including trajectory balance, a variational-inference view of GFlowNets, maximum-entropy formulations, entropy-regularized RL connections, and improved training from partial episodes~\citep{malkin2022tb,malkin2022gflownetsvi,mohammadpour2024maximum,tiapkin2024generative,madan2023learning}, with applications extending to structured decision problems such as robust scheduling~\citep{robust-scheduling}. In language modeling, amortized-inference perspectives connect GFlowNet-style objectives to reasoning over latent chains of thought~\citep{hu2024amortizing}, and closely related ideas have recently been extended to diffusion-model priors in vision, language, and control~\citep{venkatraman2024amortizing}. The closest prior work for LLM post-training is FOR and FlowRL~\citep{yu2025flow,zhu2026flowrl}. Our method stays in this distribution-matching line, but departs from FlowRL by removing the learned prompt-conditional partition network and replacing it with a rollout-based estimate motivated by the instability of jointly training that auxiliary model with a frontier-scale policy. Compared to reward-maximizing methods such as GRPO, the GFlowNet objective targets a reference-anchored \textit{distribution} rather than expected reward, retaining mass on multiple high-reward reasoning paths instead of concentrating on a single mode.

\noindent\textbf{In-batch estimation and variance reduction.}
Our estimator is also related to a broader line of work on reducing variance in score-function and policy-gradient optimization. Classical REINFORCE baselines~\citep{williams1992simple,weaver2001optimal}, action-dependent baselines~\citep{wu2018variance}, and leave-one-out or learned control variates for variational inference and black-box gradient estimation~\citep{richter2020vargrad,Tucker2017REBARLU,Grathwohl2018BackpropagationTT} all reduce estimator noise without changing the underlying objective. Path-space log-variance objectives further provide low-variance training signals for sampling problems~\citep{nusken2020hjb}, and have been shown to coincide with the trajectory-balance loss for GFlowNets~\citep{malkin2022gflownetsvi}, providing additional motivation for variance-aware partition estimates in our setting. In LLM RL, GRPO and REINFORCE++ stabilize critic-free updates through group-relative normalization and global advantage normalization~\citep{shao2024deepseekmath,hu2025reinforcepp}, while DAPO emphasizes system-level stability in large-scale reasoning training~\citep{yu2025dapo}. Our use of the rollout group is related in spirit to these in-batch estimators, but the quantity we estimate is different: we use the current group to estimate the partition term implied by the distribution-matching objective and inject it as a stop-gradient baseline, structurally analogous to a GRPO group-mean baseline applied to the trajectory-balance residual rather than to the advantage.

\noindent\textbf{Reasoning and safety domains.}
Most work on RL post-training for LLMs is evaluated on domains with verifiable rewards, especially mathematical reasoning and code generation. Representative math resources include step-level verification and olympiad-style evaluation~\citep{lightman2023let,he2024olympiadbench}, while code evaluation commonly uses LiveCodeBench, Codeforces, and HumanEval+~\citep{jain2024livecodebench,penedo2025codeforces,chen2021evaluating}. We additionally study LLM red-teaming, where rewards are sparser and noisier than in standard reasoning benchmarks. This setting connects our work to jailbreak benchmarks such as AdvBench and HarmBench and to recent automated attack methods including SEMA, FITD, and X-Teaming~\citep{zou2023universal,mazeika2024harmbench,feng2026sema,weng2025footinthedoormultiturnjailbreakllms,rahman2025xteamingmultiturnjailbreaksdefenses}. We use this domain not primarily to propose a new attack method, but to test whether \method remains stable when reward feedback is substantially noisier than in math and code.

\section{Detailed Experimental Setup}
\label{app:setup}


\begin{table*}[!h]
\centering
\small
\setlength{\tabcolsep}{6pt}
\renewcommand{\arraystretch}{1.1}
\caption{\textbf{Dataset statistics.} Training and evaluation datasets used in our experiments, spanning mathematical reasoning, code generation, and AI security.}
\label{tab:data_stats}
\begin{tabular}{l|l|l|c|l}
\toprule[1.25pt]
\textbf{Domain} & \textbf{Split} & \textbf{Dataset} & \textbf{\# Examples} & \textbf{License}\\
\midrule[1.1pt]
\multirow{7}{*}{\textbf{Math}}
  & Train & DAPO-Math-17K~\citep{DapoMath}             & 17{,}398 & Apache 2.0 \\
  \cmidrule(lr){2-5}
  & \multirow{6}{*}{Test}
          & AIME 2024~\citep{AIME24}                       & 30 & Apache 2.0 \\
  &       & AIME 2025~\citep{AIME25}                       & 30 & Apache 2.0 \\
  &       & AMC 2023~\citep{AMC}                         & 40 & None\\
  &       & MATH-500~\citep{lightman2023let,math500}             & 500 & MIT\\
  &       & Minerva~\citep{minerva,minervamath}                      & 272 & MIT \\
  &       & Olympiad~\citep{he2024olympiadbench}         & 674 & MIT \\
\midrule
\multirow{4}{*}{\textbf{Code}}
  & Train & DeepCoder~\citep{deepcoder2025}              & 24{,}000 & MIT \\
  \cmidrule(lr){2-5}
  & \multirow{3}{*}{Test}
          & LiveCodeBench~\citep{jain2024livecodebench,lcbdata}  & 279  & MIT \\
  &       & Codeforces~\citep{penedo2025codeforces}      & 408 & Apache 2.0 \\
  &       & HumanEval+~\citep{chen2021evaluating}        & 163 & Apache 2.0\\
\midrule
\multirow{2}{*}{\textbf{Security}}
  & Train & AdvBench~\citep{zou2023universal}             & 418 & MIT \\
  \cmidrule(lr){2-5}
  & \multirow{2}{*}{Test}
          & AdvBench~\citep{zou2023universal}            & 520 & MIT \\
  &       & HarmBench~\citep{mazeika2024harmbench}       & 159 & MIT \\
\bottomrule[1.25pt]
\end{tabular}
\end{table*}
\begin{table}[!h]
\centering
\caption{Detailed training Hyperparameters of \method, FlowRL and GRPO.}
\label{tab:hyperparams}
\begin{tabular}{@{}ll@{}}
\toprule
\textbf{Hyperparameter} & \textbf{Value} \\
\midrule
\multicolumn{2}{@{}l}{\textit{Optimization}} \\
\quad Optimizer & AdamW \\
\quad Learning rate & $1 \times 10^{-6}$ \\
\quad LR schedule & Constant with linear warmup \\
\quad Warmup steps & 10 \\
\quad Weight decay & 0.1 \\
\quad Gradient clipping & 1.0 \\
\midrule
\multicolumn{2}{@{}l}{\textit{Policy Optimization}} \\
\quad Advantage estimator & GRPO \\
\quad Clip ratio $\varepsilon_{\text{low}}$ & 0.2 \\
\quad Clip ratio $\varepsilon_{\text{high}}$ & 0.28 \\
\quad Loss aggregation & Token-mean \\
\quad KL coefficient & 0.0 \\
\quad Entropy coefficient & 0.0 \\
\midrule
\multicolumn{2}{@{}l}{\textit{Rollout}} \\
\quad Responses per prompt $G$ & 16 \\
\quad Temperature (training) & 1.0 \\
\quad Top-$p$ (training) & 1.0 \\
\quad Temperature (evaluation) & 1.0 \\
\quad Top-$p$ (evaluation) & 0.7 \\
\quad Overlong reward shaping & Enabled \\
\quad Overlong penalty factor & 1.0 \\
\quad Importance sampling & Sequence-level \\
\quad IS threshold & 2.0 \\
\quad Ratio scaling & Geometric \\
\quad RS threshold (upper / lower) & 1.01 / 0.99 \\
\quad Filter groups & Accuracy-based \\
\bottomrule
\end{tabular}
\end{table}

\begin{table}[t]
\centering
\caption{Training configurations for math tasks.}
\label{tab:math-configs}
\small
\begin{tabular}{l c c c c c c c}
\toprule
\textbf{Model} 
& $L_{\text{ctx}}$ 
& $L_{\text{prompt}}$ 
& \textbf{GPUs} 
& \textbf{Strat.} 
& $B$ 
& $b$ 
& \textbf{Steps} \\
\midrule

Qwen2.5-7B    
& $8\text{k}$  & $2\text{k}$ & $1\times8$ & FSDP2    & 512 & 64 & $[0,300]$ \\

Qwen2.5-32B   
& $8\text{k}$  & $2\text{k}$ & $4\times8$ & FSDP2    & 512 & 64 & $[0,300]$ \\

Qwen3-30B-A3B 
& $32\text{k}$ & $2\text{k}$ & $4\times8$ & Megatron & 512 & 32 & $[0,100]$ \\

\bottomrule
\end{tabular}
\end{table}
\begin{table*}[!h]
\centering
\caption{Per-model training configurations following a two-stage schedule, as in DeepCoder~\citep{deepcoder2025}. Models are selected based on validation performance.}
\label{tab:training-configs-deepcoder}
\small
\setlength{\tabcolsep}{2pt}
\begin{tabular}{llcccccccc}
\toprule
\textbf{Model} &  \textbf{Stage} &\textbf{Context Len.}  & \textbf{Prompt Len.}
& \textbf{GPUs}
&\textbf{Train BSZ}
& \textbf{Mini BSZ}
& \textbf{Steps} \\
\midrule
\multirow{2}{*}{DeepSeek-R1-Distill-Qwen-7B}
& I  & 16k & 3k & $4\times8$ & 128 & 32 & 0--200 \\
& II & 32k & 3k & $4\times8$ & 128 & 32 & 200--400 \\

\addlinespace
\multirow{2}{*}{DeepSeek-R1-Distill-Qwen-14B}
& I  & 16k & 3k & $4\times8$ & 128 & 64 & 0--200 \\ 
& II & 32k & 3k & $4\times8$ & 128 & 64 & 200--400 \\

\addlinespace
\multirow{2}{*}{Qwen3-30B-A3B}
& I  & 16k & 2k & $4\times8$ & 128 & 32 & 0--200 \\
& II & 32k & 2k & $8\times8$ & 128 & 32 & 200--400 \\

\bottomrule
\end{tabular}
\end{table*}
\begin{table}[!h]
\centering
\caption{Training configurations for code tasks.}
\label{tab:code-configs}
\small
\begin{tabular}{l c c c c c c c}
\toprule
\textbf{Model} 
& $L_{\text{ctx}}$ 
& $L_{\text{prompt}}$ 
& \textbf{GPUs} 
& \textbf{Strat.} 
& $B$ 
& $b$ 
& \textbf{Steps} \\
\midrule

DeepSeek-R1-Distill-Qwen-7B    
& $8\text{k}$  & $2\text{k}$ & $1\times8$ & FSDP2    & 512 & 64 & $[0,300]$ \\

DeepSeek-R1-Distill-Qwen-14B   
 & $8\text{k}$  & $2\text{k}$ & $4\times8$ & FSDP2    & 512 & 64 & $[0,300]$ \\

\bottomrule
\end{tabular}
\end{table}
\begin{table}[!h]
\centering
\caption{Training configurations of MoE for math tasks.}
\label{tab:moe-configs}
\small
\begin{tabular}{l c c c c c c c}
\toprule
\textbf{Model} 
& $L_{\text{ctx}}$ 
& $L_{\text{prompt}}$ 
& \textbf{GPUs} 
& \textbf{Strat.} 
& $B$ 
& $b$ 
& \textbf{Steps} \\
\midrule

Qwen3-30B-A3B    
& $32\text{k}$  & $2\text{k}$ & $4\times8$ & Megatron    & 256 & 64 & $[0,100]$ \\

Qwen3-235B-A22B   
 & $32\text{k}$  & $2\text{k}$ & $8\times8$ & Megatron    & 128 & 32 & $[0,100]$ \\

\bottomrule
\end{tabular}
\end{table}
\subsection{Backbone Models}
For mathematical reasoning we use the dense pretrained \textit{Qwen2.5-7B} and \textit{Qwen2.5-32B}~\citep{qwen2.5}. For coding we use \textit{Deepseek-R1-Distilled-Qwen-7B} and \textit{Deepseek-R1-Distilled-Qwen-14B}~\citep{guo2025deepseek}. To assess scalability to sparse architectures, we further evaluate the mixture-of-experts models \textit{Qwen3-30B-A3B} (3B active parameters) and \textit{Qwen3-235B-A22B} (22B active parameters).

\subsection{Datasets}
Per-dataset statistics are summarized in Table~\ref{tab:data_stats}.
\begin{itemize}
    \item \textit{Math:} DAPO-Math-17K~\citep{DapoMath}, AIME2024~\citep{AIME24}, AIME2025~\citep{AIME25}, AMC2023~\citep{AMC}, MATH-500~\citep{lightman2023let}, Minerva~\citep{minerva}, and Olympiad~\citep{he2024olympiadbench}.
    \item \textit{Coding:} LiveCodeBench~\citep{jain2024livecodebench}, Codeforces~\citep{penedo2025codeforces}, and HumanEval+~\citep{chen2021evaluating}.
    \item \textit{Security:} We jailbreak target LLMs with SEMA~\citep{feng2026sema} on AdvBench~\citep{zou2023universal} and HarmBench~\citep{mazeika2024harmbench}.
\end{itemize}

\subsection{Metrics}
Evaluation metrics follow standard practice in each domain. For math, we report Avg@16 accuracy~\citep{yu2025dapo, zhu2026flowrl}. For Codeforces, we report competitive-programming Rating and Percentile, following DeepCoder~\citep{deepcoder2025}. For security, we report attack success rate ASR@1, following GCG~\citep{zou2023universal, feng2026sema}.

\subsection{Baselines}
We organize baselines into two families: (1) \textit{distribution-matching}, FlowRL~\citep{zhu2026flowrl}; and (2) \textit{reward maximization}, REINFORCE++ (R++)~\citep{hu2025reinforcepp}, PPO~\citep{schulman2017proximal}, and GRPO/GRPO+~\citep{shao2024deepseekmath, deepcoder2025}. The exact subset used in each experiment is noted in the corresponding results section.

\subsection{Implementation}
Our implementation is built on \textit{verl}\footnote{\url{https://github.com/verl-project/verl}}, \textit{deepcoder}\footnote{\url{https://github.com/agentica-project/rllm}}, \textit{FlowRL}\footnote{\url{https://github.com/Xuekai-Zhu/FlowRL}}, and \textit{SEMA}\footnote{\url{https://github.com/fmmarkmq/SEMA}} in PyTorch. Unless otherwise specified, we adopt each framework's default settings for batch size, learning rate, and number of training steps. The full hyperparameters and training configurations are reported in Tables~\ref{tab:hyperparams}, \ref{tab:math-configs}, \ref{tab:code-configs}, \ref{tab:training-configs-deepcoder}, and \ref{tab:moe-configs}.

\subsection{Diagnostic and Synthetic Setup}
\label{app:diagnostic_setup}
\noindent\textbf{Random-$\log Z$ replacement.} For the diagnostic in Sec.~\ref{subsec:diagnostic}, we replace $\log Z_\phi(\mathbf{x})$ at every training step with an independent sample $\log Z \sim \mathcal{N}(\mu, \sigma^2)$, where $(\mu, \sigma^2)$ are matched to the empirical distribution of $\log Z_\phi(\mathbf{x})$ observed under standard FlowRL training~\citep{zhu2026flowrl}. We use $\mu = 0.5$ and $\sigma = 1$. All other settings (optimizer, batch size, learning rate, sampling temperature) are kept identical to the FlowRL baseline.

\noindent\textbf{Synthetic distribution-matching benchmark.} We construct a synthetic target as a mixture of three Gaussian components with distinct centers in a 100-dimensional space. Although this is far smaller than realistic LLM settings, it preserves the few-hundred-step optimization regime characteristic of LLM post-training and isolates distribution-matching behavior from confounders such as model capacity. We compare GRPO, FlowRL, FlowRL-RandomLogZ, and \method at training steps $\{100, 500, 1000, 5000\}$; full training dynamics are reported in Appendix~\ref{app:dmc}.

\subsection{Per-Experiment Training Configurations}
\label{app:per_exp_configs}
\noindent\textbf{Codeforces (dense, 14B).} For the frontier comparison in Figure~\ref{fig:cf-comparison}, \method is trained on the same data and two-stage curriculum released by DeepCoder~\citep{deepcoder2025}, and the full configuration in provided in Table~\ref{tab:training-configs-deepcoder}. Due to compute constraints, we train \method for 400 total steps versus 600 for DeepCoder, providing a fair, but under-trained comparison.

\noindent\textbf{MoE training.} For Qwen3-30B-A3B and Qwen3-235B-A22B, we apply \method without modification across scales: the hyperparameters, data, and training schedule used at 30B are reused at 235B. This shared configuration is listed in Table~\ref{tab:moe-configs}.

\section{Training Curve}
\label{app:training}
\begin{figure*}[ht]
    \centering
    \includegraphics[width=1\linewidth]{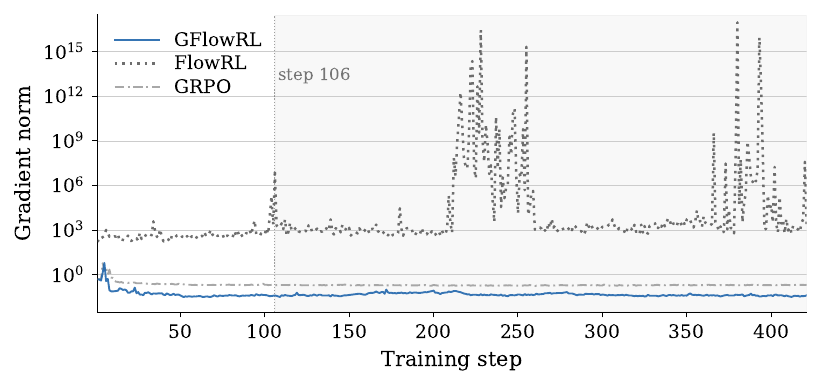}
    \caption{Gradient norm over training steps on Qwen2.5-7B math reasoning tasks. FlowRL suffers 55 gradient explosions $\geq 10^6$ beginning at step 106, while GFlowRL and GRPO remain stable throughout.}
    \label{fig:grad_norm}    
\end{figure*}

The training curves, as shown in Figure~\ref{fig:grad_norm}, further corroborate this instability: unlike GRPO and GFlowRL, whose gradient norms remain relatively smooth throughout optimization, FlowRL displays repeated, abrupt spikes spanning many orders of magnitude. These excursions are not isolated numerical artifacts but occur persistently across training, with $55$ of $421$ steps exceeding a gradient norm of $10^6$. Such behavior indicates that the learned normalization term $Z_\phi$ destabilizes policy updates, allowing poorly calibrated estimates of $\log Z_\phi(\mathbf{x})$ to dominate the optimization signal despite providing no measurable downstream benefit.

\section{Additional Results}
\label{app:additional_results}
In this appendix, we provide additional results which demonstrate the superior performance of \method.

\subsection{Extended Comparison Against Additional Baselines}
\label{app:comp_logs}
\begin{table*}[!h]
  \centering
  \caption{\textbf{Effect of $\log Z$ estimation method on math reasoning.} Average accuracy across six math reasoning benchmarks (i.e., the Test Math Benchmarks in Table~\ref{tab:data_stats}) with Qwen2.5-7B as the backbone.}
  \label{tab:logz_comp}
  \vspace{5pt}
  \begin{tabular}{l|cccc}
    \toprule[1.25pt]
    \textbf{Method} & Backbone & ConstantLogZ & FOR & GFlowRL \\
    \midrule[1.1pt]
    Average Acc & 23.02 & 25.34 & 29.53 & 40.92 \\
    \bottomrule[1.25pt]
  \end{tabular}
\end{table*}
As shown in Table~\ref{tab:logz_comp}, \method (40.92) substantially outperforms all alternatives, nearly doubling the backbone's accuracy (23.02) and beating the next-best method, FOR (29.53), by 11.4 points. ConstantLogZ (25.34) barely improves over the backbone, indicating that naive $\log Z$ handling is a real bottleneck. Proper estimation is what drives the gain.

\subsection{Dense 32B Math Results}
\label{app:math_32b}
\begin{table*}[ht]
  \centering
  \caption{\textbf{Results on math reasoning benchmarks for dense models.} We report Avg@16 accuracy with relative improvements over the backbone shown as subscripts. Positive gains are shown in \textcolor{scoregreen}{green} and negative changes in \textcolor{red}{red}. \method outperforms or is competitive with all baselines on the 32B model scale.}
\label{tab:math_results}
  \resizebox{\textwidth}{!}{ 
  \begin{tabular}{l|llllll|l}
      \toprule[1.25pt]
        \textbf{Models} 
        & \textbf{AIME24} 
        & \textbf{AIME25}
        & \textbf{AMC23}
        & \textbf{MATH500}
        & \textbf{Minerva}
        & \textbf{Olympiad}
        & \textbf{Avg} \\
      \midrule[1.1pt]  
      \rowcolor[rgb]{0.93,0.93,0.93}
      \multicolumn{8}{c}{\texttt{Qwen2.5-32B-Base, Max Response Len = 8K tokens}} \\
      \textcolor{gray}{\texttt{Backbone}} & \textcolor{gray}{4.58} & \textcolor{gray}{2.08} & \textcolor{gray}{28.59} & \textcolor{gray}{52.48} & \textcolor{gray}{26.99} & \textcolor{gray}{21.37} & \textcolor{gray}{22.68} \\
      R++ & $14.79_{\textcolor{scoregreen}{+10.21}}$ & $9.17_{\textcolor{scoregreen}{+7.09}}$ & $52.65_{\textcolor{scoregreen}{+24.06}}$ & $44.35_{\textcolor{red}{-8.13}}$ & $17.37_{\textcolor{red}{-9.62}}$ & $24.52_{\textcolor{scoregreen}{+3.15}}$ & 27.14 \\
      PPO & $26.87_{\textcolor{scoregreen}{+22.29}}$ & $20.41_{\textcolor{scoregreen}{+18.33}}$ & $76.40_{\textcolor{scoregreen}{+47.81}}$ & $69.17_{\textcolor{scoregreen}{+16.69}}$ & $28.79_{\textcolor{scoregreen}{+1.80}}$ & $37.90_{\textcolor{scoregreen}{+16.53}}$ & 43.26 \\
      GRPO & $23.12_{\textcolor{scoregreen}{+18.54}}$ & $14.58_{\textcolor{scoregreen}{+12.50}}$ & $76.87_{\textcolor{scoregreen}{+48.28}}$ & $61.60_{\textcolor{scoregreen}{+9.12}}$ & $18.95_{\textcolor{red}{-8.04}}$ & $34.94_{\textcolor{scoregreen}{+13.57}}$ & 38.34 \\
      FlowRL & $23.95_{\textcolor{scoregreen}{+19.37}}$ & $\mathbf{21.87}_{\textcolor{scoregreen}{+19.79}}$ & $73.75_{\textcolor{scoregreen}{+45.16}}$ & $80.75_{\textcolor{scoregreen}{+28.27}}$ & $38.21_{\textcolor{scoregreen}{+11.22}}$ & $51.83_{\textcolor{scoregreen}{+30.46}}$ & 48.39 \\
    FlowRL-RandomLogZ & $\mathbf{30.42}_{\textcolor{scoregreen}{+25.84}}$ & $16.46_{\textcolor{scoregreen}{+14.38}}$ & $\mathbf{79.38}_{\textcolor{scoregreen}{+50.79}}$ & $72.89_{\textcolor{scoregreen}{+20.41}}$ & $27.78_{\textcolor{scoregreen}{+0.79}}$ & $41.46_{\textcolor{scoregreen}{+20.09}}$ & 44.73 \\      \midrule
      \method & $24.17_{\textcolor{scoregreen}{+19.59}}$ & $18.75_{\textcolor{scoregreen}{+16.67}}$ & $76.25_{\textcolor{scoregreen}{+47.66}}$ & $\mathbf{86.31}_{\textcolor{scoregreen}{+33.83}}$ & $\mathbf{42.97}_{\textcolor{scoregreen}{+15.98}}$ & $\mathbf{54.08}_{\textcolor{scoregreen}{+32.71}}$ & $\mathbf{50.42}$ \\
      \bottomrule[1.25pt]
  \end{tabular}}
\end{table*}
Table~\ref{tab:math_results} further evaluates the 32B setting. Consistent with the main results, \method achieves the best overall average score (\textit{49.40}), improving over PPO, GRPO, and FlowRL. Compared with FlowRL, \method gains $+1.01$ average points and obtains the best results on MATH500 and Minerva, suggesting stronger generalization on broader math reasoning tasks.

We also include FlowRL-RandomLogZ as an ablation. While it performs well on AIME24 and AMC23, its lower scores on AIME25, Minerva, and Olympiad lead to a lower average than both FlowRL and \method. This suggests that randomizing the log-normalizer can help on some benchmarks but is less stable across the full set of math benchmarks.
\definecolor{promptbg}{HTML}{F5F7FA}
\definecolor{promptframe}{HTML}{1F5FAF}
\definecolor{prompttitle}{HTML}{143E73}
\definecolor{prompttext}{HTML}{5F5F5F}

\newtcolorbox{promptbox}[1][]{
  colback=promptbg,
  colframe=promptframe,
  coltitle=prompttitle,
  coltext=prompttext,
  coltitle=white,  
  boxrule=0.6pt,
  arc=2pt,
  left=6pt, right=6pt, top=6pt, bottom=6pt,
  fonttitle=\bfseries\small,
  fontupper=\ttfamily\footnotesize,
  breakable,
  #1
}

\begin{figure}[th]
\centering
\begin{promptbox}[title={Diversity Judge Prompt}, height=0.7\textheight]
\begin{verbatim}
System: You are evaluating the DIVERSITY of solution approaches for a
mathematics competition problem. Focus on detecting even SUBTLE differences
in methodology that indicate different problem-solving strategies.

PROBLEM:
{{ problem }}

16 SOLUTION ATTEMPTS:
{{ formatted_responses }}

EVALUATION CRITERIA - Rate diversity from 1 to 5:

Score 1 - Minimal Diversity:
- 14+ responses use essentially identical approaches
- Same mathematical setup, same variable choices, same solution path
- Only trivial differences (arithmetic, notation, wording)
- Indicates very low exploration/diversity in the generation process

Score 2 - Low Diversity:
- 11-13 responses use the same main approach
- 1-2 alternative approaches appear but are rare
- Minor variations within the dominant method (different substitutions,
  orderings)
- Some exploration but heavily biased toward one strategy

Score 3 - Moderate Diversity:
- 7-10 responses use the most common approach
- 2-3 distinct alternative approaches present
- Noticeable variation in problem setup or mathematical techniques
- Balanced mix showing reasonable exploration

Score 4 - High Diversity:
- 4-6 responses use the most common approach
- 3-4 distinct solution strategies well-represented
- Multiple mathematical techniques and problem framings
- Strong evidence of diverse exploration strategies

Score 5 - Maximum Diversity:
- No single approach dominates (<=3 responses use same method)
- 4+ distinctly different solution strategies
- Wide variety of mathematical techniques and creative approaches
- Excellent exploration and generation diversity

IMPORTANT: Focus on the DIVERSITY of the attempted approaches. Return ONLY
a number from 1 to 5.
\end{verbatim}
\end{promptbox}
\caption{Prompt used for the LLM-as-judge diversity evaluation.}
\label{fig:diversity-prompt}
\end{figure}

\subsection{Full results on adversarial red-teaming}
\label{app:sema}
\begin{table*}[ht]
  \centering
  \caption{\textbf{Comparison of ASR@1 $\uparrow$ for victim models on \textit{AdvBench} (\textit{LLM Classifier}) and \textit{HarmBench} (\textit{HarmBench Classifier}) for dense models.} All victim models are the instruction-tuned version rather than the base model, while we omitted the ``Instruct'' suffix for simplicity. FlowRL fails to converge.}
  \vspace{5pt}
  \label{tab:sema_main_result}
  \resizebox{\textwidth}{!}{
  \begin{tabular}{l|llll|llll}
    \toprule[1.25pt]
    \textbf{Attackers / Victim Models}
    & \multicolumn{4}{c|}{\textbf{AdvBench} \citep{zou2023universal}}
    & \multicolumn{4}{c}{\textbf{HarmBench} \citep{mazeika2024harmbench}} \\
    \cmidrule(lr){2-5} \cmidrule(lr){6-9}
    & Qwen2.5-3B & Llama-3.1-8B & GPT-4.1-mini & Avg 
    & Qwen2.5-3B & Llama-3.1-8B & GPT-4.1-mini & Avg \\
    \midrule[1.1pt]
    \rowcolor[rgb]{0.93,0.93,0.93}
    \multicolumn{9}{c}{\texttt{Single-turn}} \\
    FlipAttack \citep{liu2024flipattackjailbreakllmsflipping}
    & 1.7 & 1.2 & 31.4 & 11.4
    & 0.0 & 1.9 & 44.7 & 15.5 \\
    ADV-LLM \citep{sun2025advllmiterativeselftuningllms}
    & 68.1 & 63.7 & 6.7 & 46.2
    & 66.7 & 69.2 & 29.6 & 55.1 \\
    Jailbreak-R1 \citep{guo2025jailbreakr1exploringjailbreakcapabilities}
    & 23.1 & 16.2 & 15.0 & 18.1
    & 30.8 & 21.4 & 15.1 & 22.4 \\
    \midrule[1.1pt]
    \rowcolor[rgb]{0.93,0.93,0.93}
    \multicolumn{9}{c}{\texttt{Multi-turn}} \\
    Jigsaw Puzzle \citep{yang2024jigsawpuzzlessplittingharmful}
    & 22.9 & 36.7 & 58.7 & 39.4
    & 17.6 & 32.7 & 62.3 & 37.5 \\
    Crescendo \citep{russinovich2025greatwritearticlethat}
    & 36.0 & 35.2 & 48.5 & 39.9
    & 40.9 & 34.0 & 47.8 & 40.9 \\
    GOAT \citep{pavlova2024automatedredteaminggoat}
    & 27.5 & 8.5 & 31.9 & 22.6
    & 22.6 & 4.4 & 29.6 & 18.9 \\
    CoA \citep{yang2024chainattacksemanticdrivencontextual}
    & 11.2 & 11.0 & 13.1 & 11.7
    & 17.6 & 12.0 & 19.5 & 16.4 \\
    FITD \citep{weng2025footinthedoormultiturnjailbreakllms}
    & 20.0 & 21.0 & 22.3 & 21.1
    & 28.3 & 23.9 & 18.2 & 23.5 \\
    ActorAttack \citep{ren2025llmsknowvulnerabilitiesuncover}
    & 8.8 & 9.2 & 13.3 & 10.4
    & 7.7 & 9.6 & 11.5 & 9.6 \\
    X-Teaming \citep{rahman2025xteamingmultiturnjailbreaksdefenses}
    & 39.4 & 24.2 & 44.2 & 36.0
    & 45.3 & 22.0 & 44.7 & 37.3 \\
    SEMA \citep{feng2026sema}
    & 79.9 & 77.2 & 83.3 & 80.1
    & 74.5 & 70.6 & 79.8 & 75.0 \\
    \midrule
    \method & \textbf{80.2} & \textbf{81.2} & \textbf{86.1} & \textbf{82.5} & \textbf{79.9} & \textbf{73.0} & \textbf{85.5} & \textbf{79.5} \\
    \bottomrule[1.25pt]
  \end{tabular}}
\end{table*}

The full results in Table~\ref{tab:sema_main_result} compare \method against a broad set of single-turn and multi-turn red-teaming baselines. \method achieves the best average ASR@1 on both AdvBench and HarmBench, outperforming prior methods across all victim models. Compared with the strongest baseline, SEMA~\citep{feng2026sema}, \method improves average ASR by $+2.4$ points on AdvBench and $+4.5$ points on HarmBench. Notably, FlowRL fails to converge in this noisier, sparse-reward setting, while \method remains stable, highlighting the benefit of removing the high-variance $\log Z_\phi$ term.

\section{Diversity Judge Prompt}
\label{app:divs}
Table~\ref{tab:diversity}  in Section~\ref{subsec:astudy} reports the diversity scores for several models, as judged by GPT-o4-mini, on a scale from 1–5. The prompt template that is used to compute the diversity scores is provided in Figure~\ref{fig:diversity-prompt}. The course grained diversity scores are computed across the 16 solutions for each problem and range from Minimal Diversity (score = 1) to Maximum Diversity (score = 5), as \textit{EVALUATION CRITERIA in} the prompt. 
\newpage
\section{Distribution Matching Comparison}
\label{app:dmc}
\begin{figure}[!h]
    \centering
    \begin{subfigure}[b]{0.48\textwidth}
        \centering
        \includegraphics[width=\textwidth]{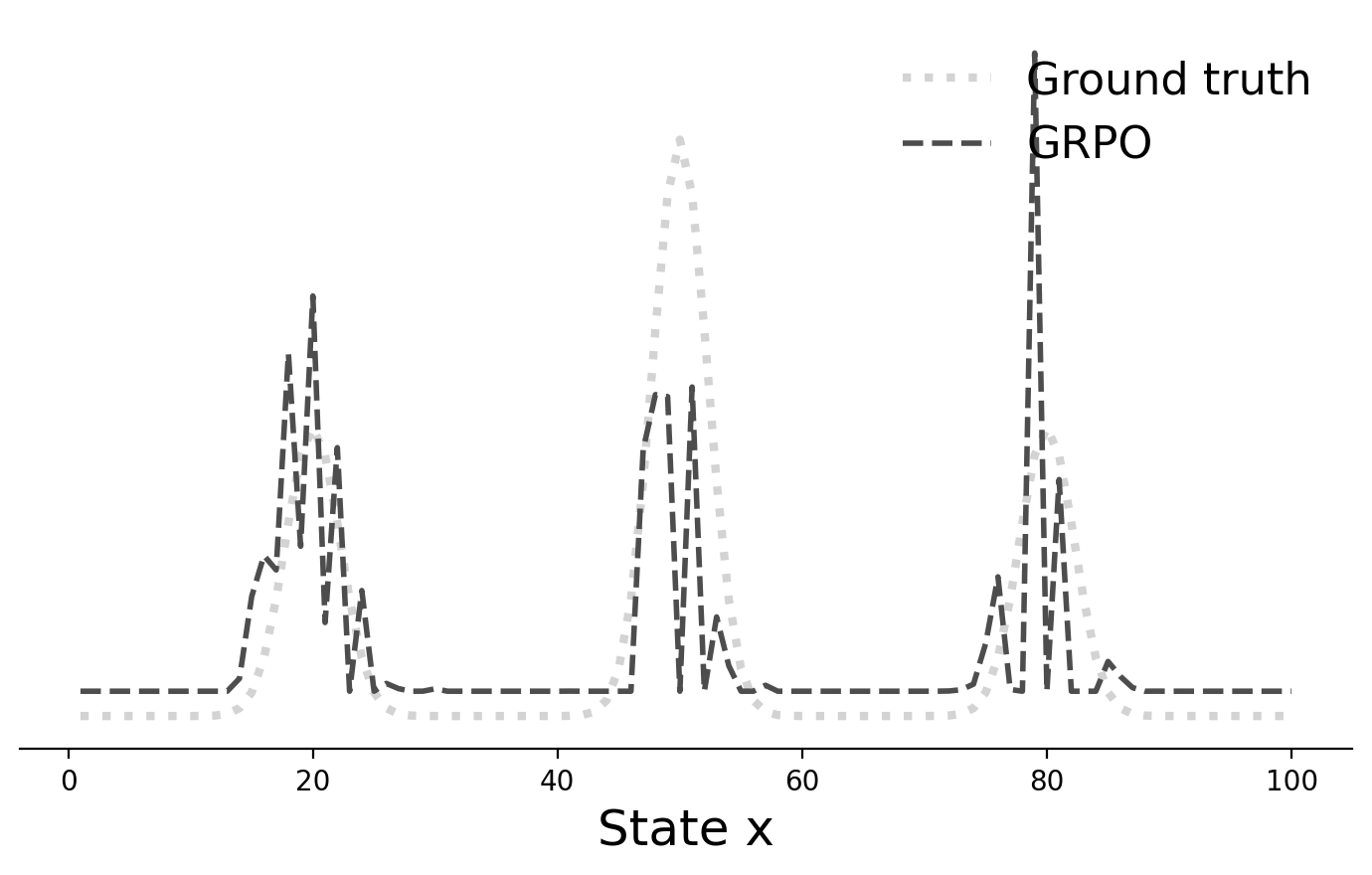}
        \caption{GRPO}
        \label{fig:grpo}
    \end{subfigure}
    \hfill
    \begin{subfigure}[b]{0.48\textwidth}
        \centering
        \includegraphics[width=\textwidth]{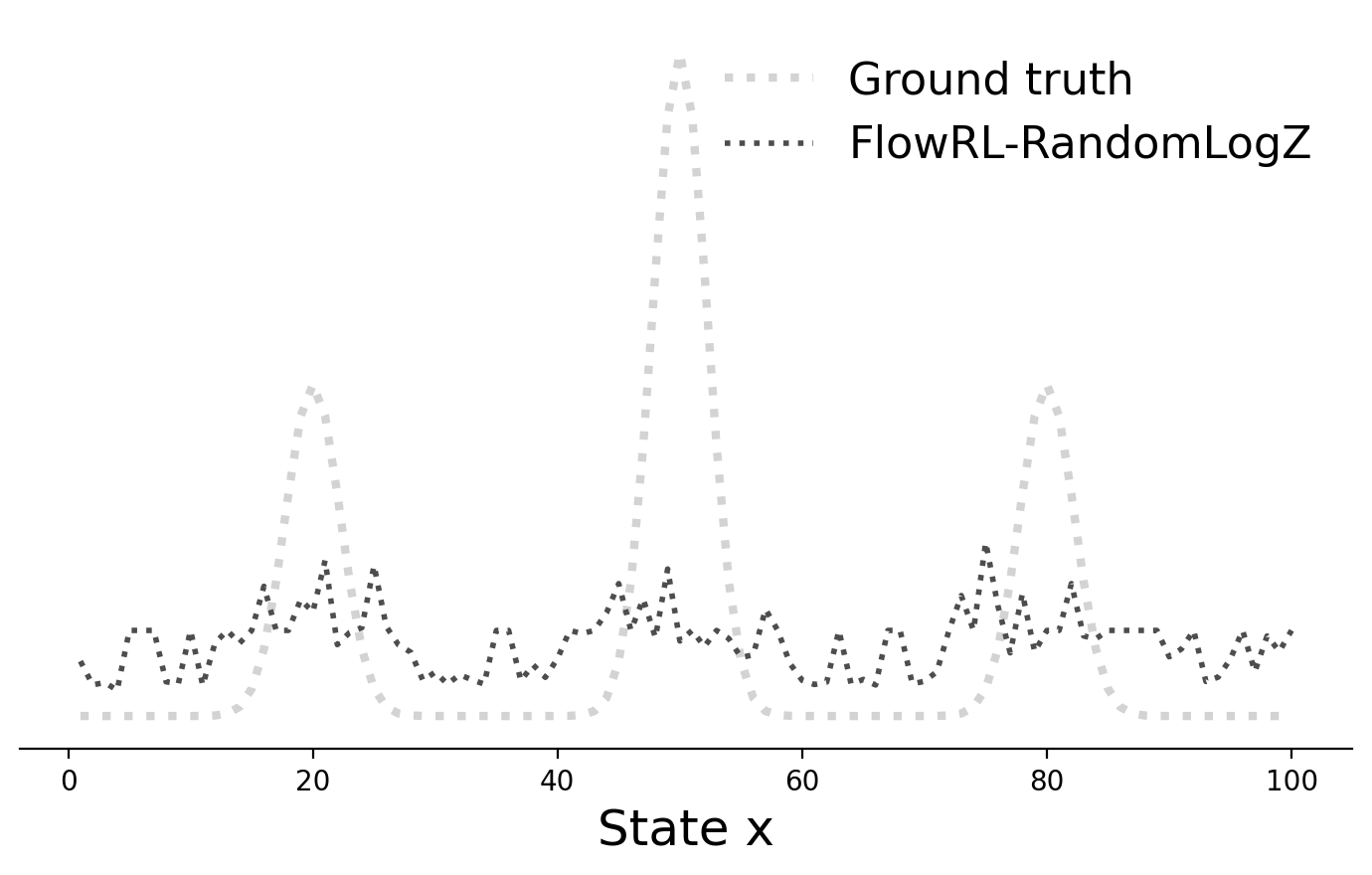}
        \caption{FlowRL-RandomLogZ}
        \label{fig:random}
    \end{subfigure}

    \vspace{0.5em} 

    \begin{subfigure}[b]{0.48\textwidth}
        \centering
        \includegraphics[width=\textwidth]{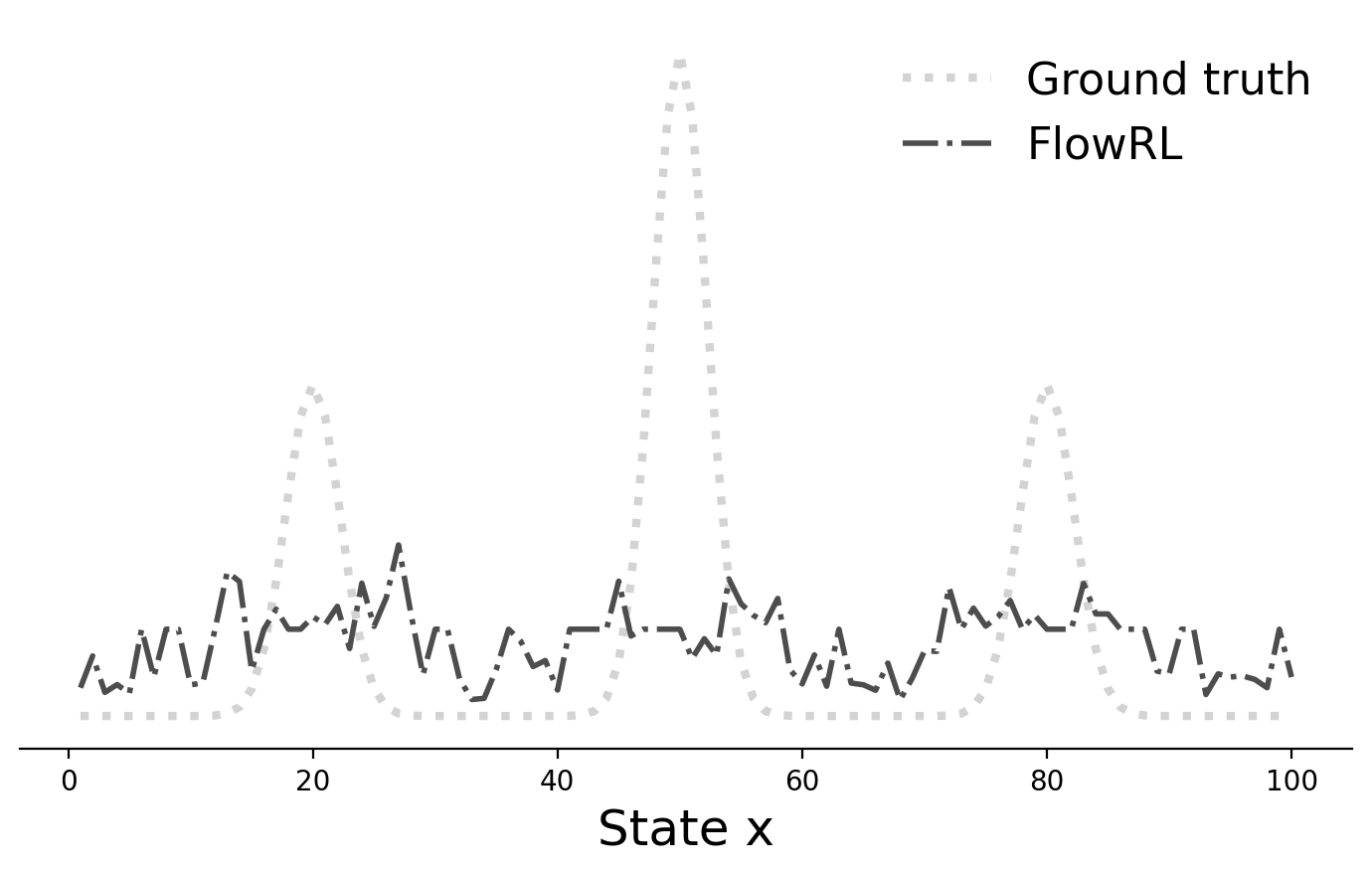}
        \caption{FlowRL}
        \label{fig:flowrl}
    \end{subfigure}
    \hfill
    \begin{subfigure}[b]{0.48\textwidth}
        \centering
        \includegraphics[width=\textwidth]{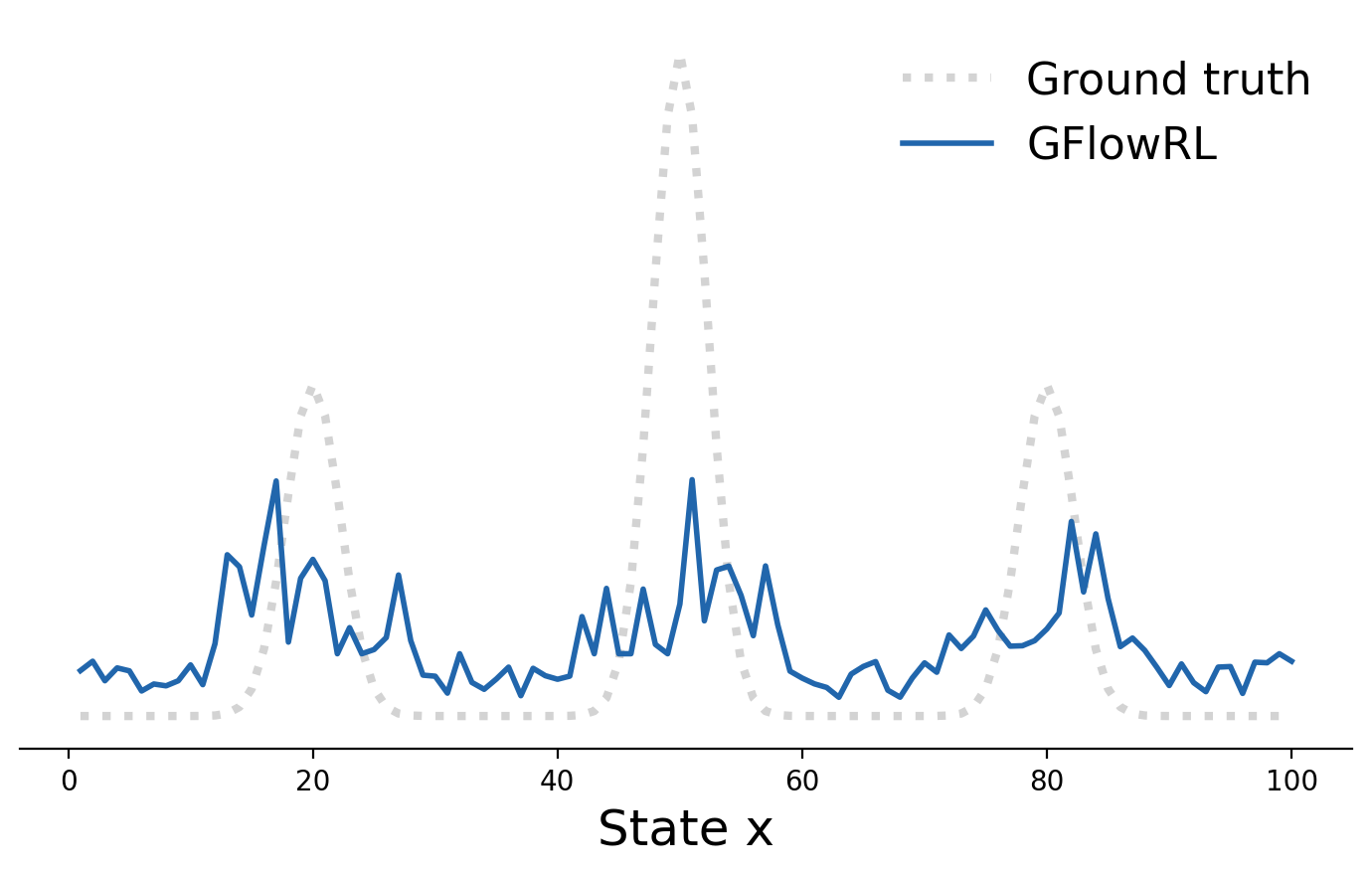}
        \caption{GFlowRL}
        \label{fig:gflowrl}
    \end{subfigure}

    \caption{Comparison of distribution matching across GRPO, FlowRL-RandomLogZ, FlowRL, and GFlowRL at the 100th training step.}
    \label{fig:step100}
\end{figure}
\begin{figure}[!h]
    \centering
    \begin{subfigure}[b]{0.48\textwidth}
        \centering
        \includegraphics[width=\textwidth]{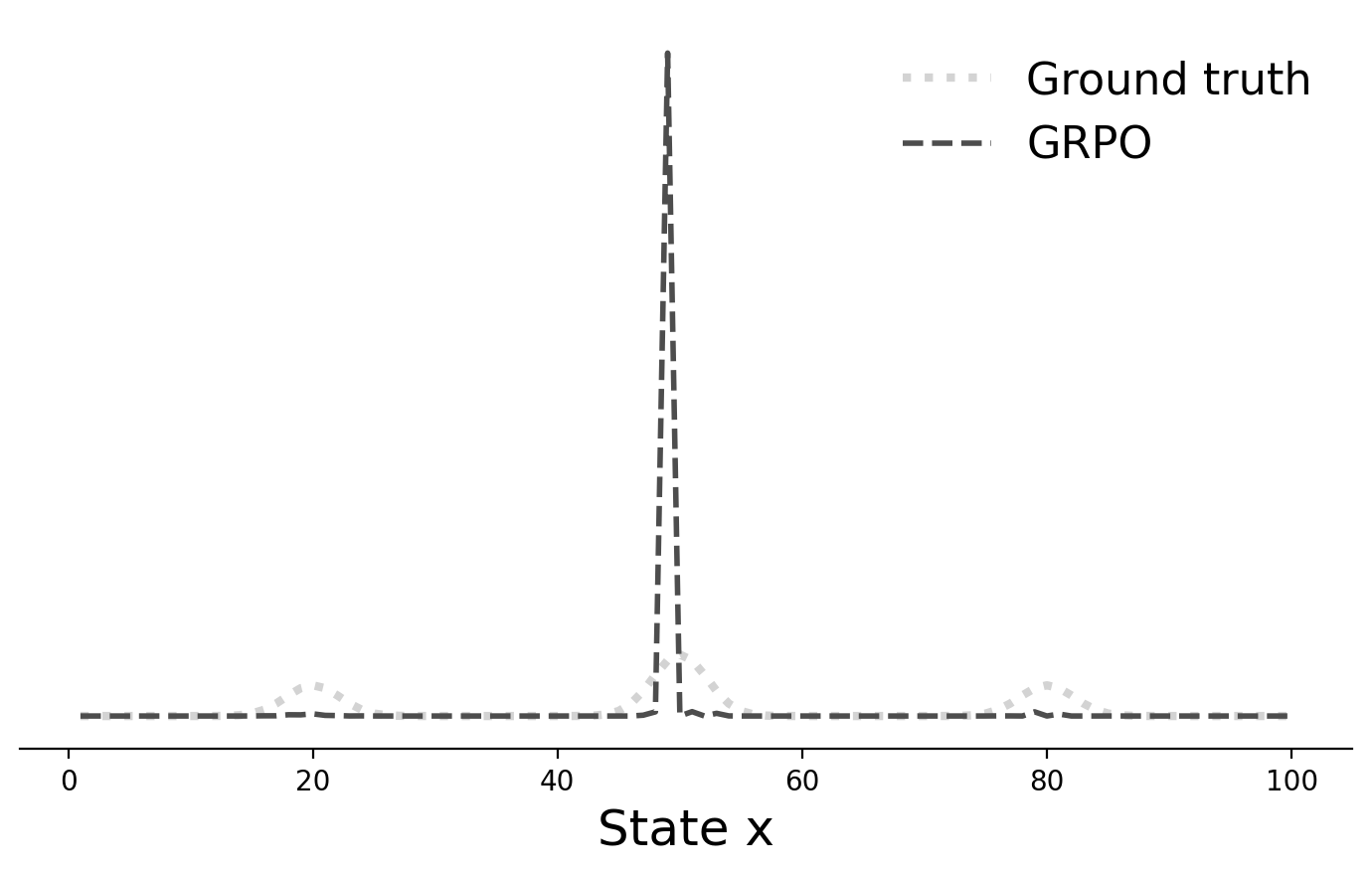}
        \caption{GRPO}
        \label{fig:grpo500}
    \end{subfigure}
    \hfill
    \begin{subfigure}[b]{0.48\textwidth}
        \centering
        \includegraphics[width=\textwidth]{figures/flowrl-randomlogz_500.png}
        \caption{FlowRL-RandomLogZ}
        \label{fig:random500}
    \end{subfigure}
    \vspace{0.5em} 
    \begin{subfigure}[b]{0.48\textwidth}
        \centering
        \includegraphics[width=\textwidth]{figures/flowrl_500.png}
        \caption{FlowRL}
        \label{fig:flowrl500}
    \end{subfigure}
    \hfill
    \begin{subfigure}[b]{0.48\textwidth}
        \centering
        \includegraphics[width=\textwidth]{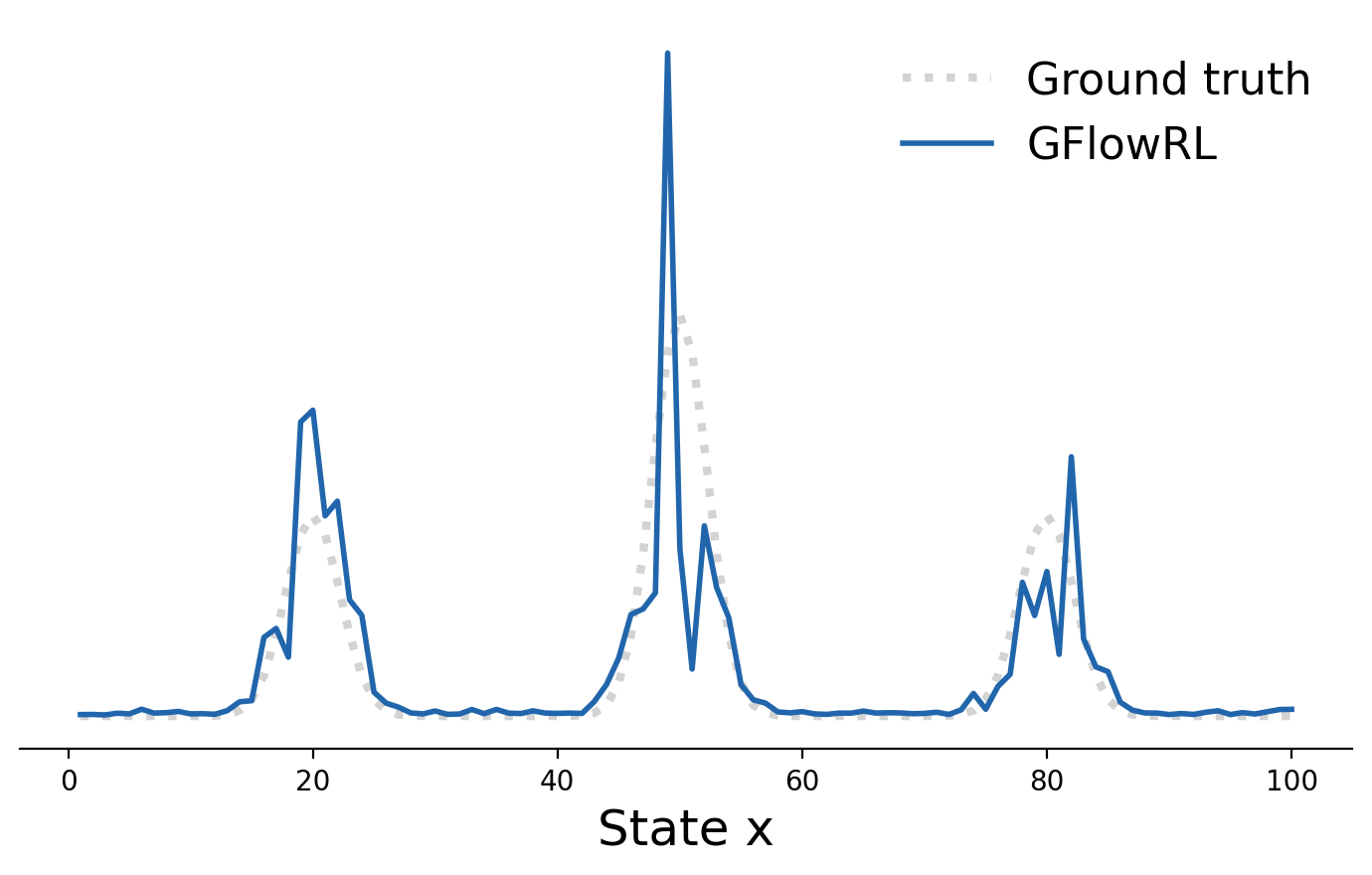}
        \caption{GFlowRL}
        \label{fig:gflowrl500}
    \end{subfigure}
    \caption{Comparison of distribution matching across GRPO, FlowRL-RandomLogZ, FlowRL, and GFlowRL at the 500th training step.}
    \label{fig:step500}
\end{figure}
\begin{figure}[!h]
    \centering
    \begin{subfigure}[b]{0.46\textwidth}
        \centering
        \includegraphics[width=\textwidth]{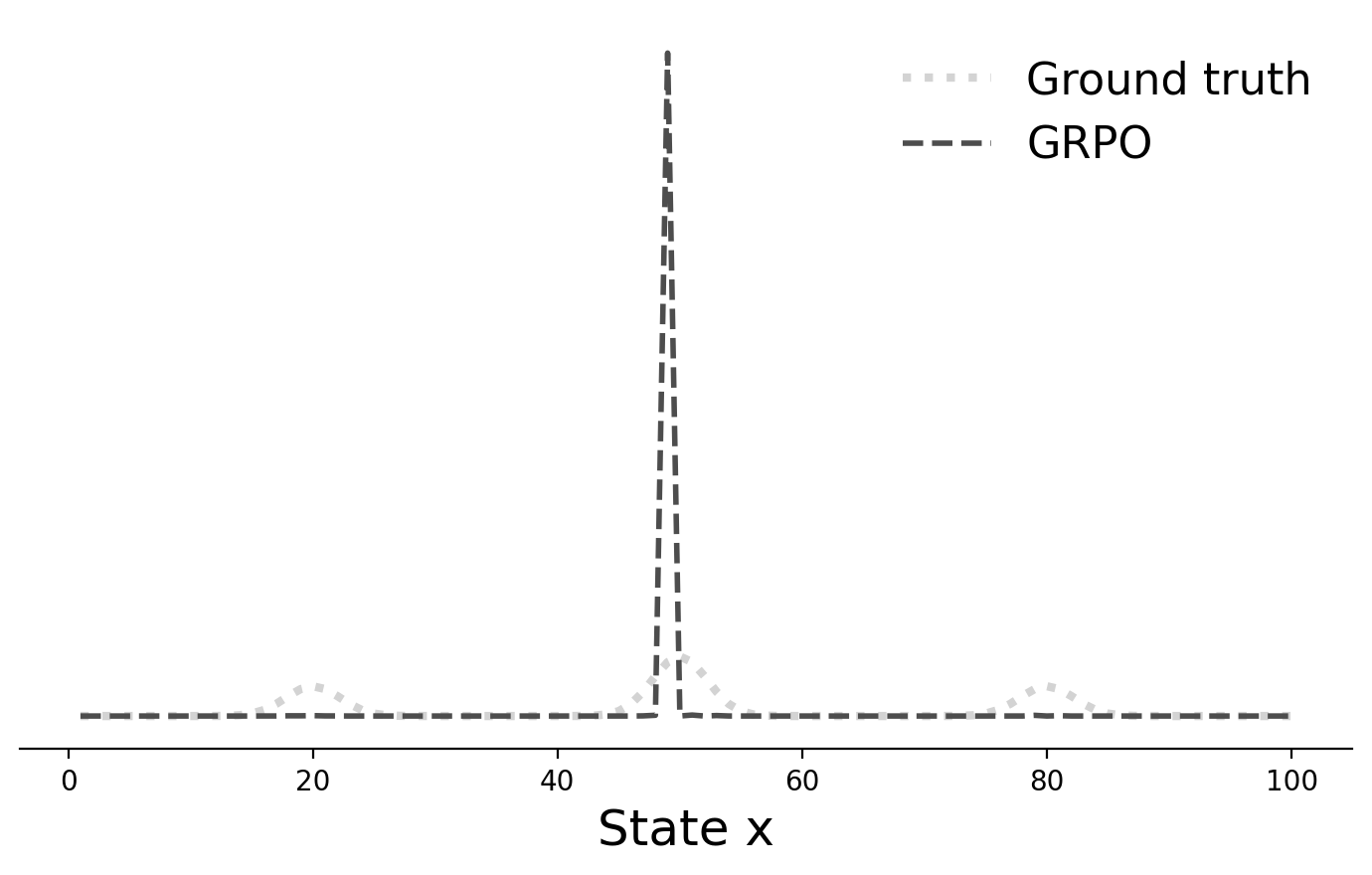}
        \caption{GRPO}
        \label{fig:grpo1000}
    \end{subfigure}
    \hfill
    \begin{subfigure}[b]{0.46\textwidth}
        \centering
        \includegraphics[width=\textwidth]{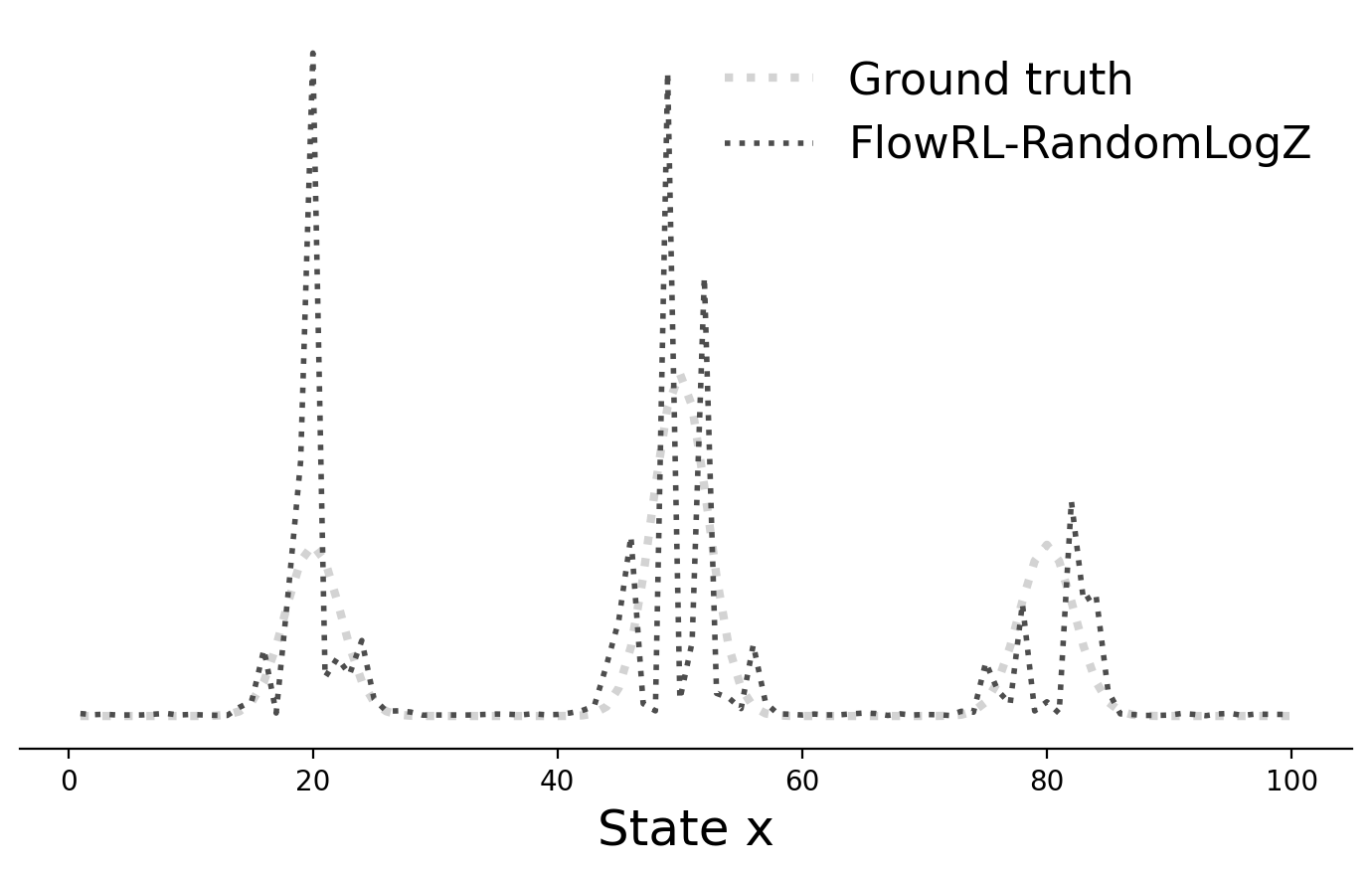}
        \caption{FlowRL-RandomLogZ}
        \label{fig:random1000}
    \end{subfigure}
    \vspace{0.5em} 
    \begin{subfigure}[b]{0.46\textwidth}
        \centering
        \includegraphics[width=\textwidth]{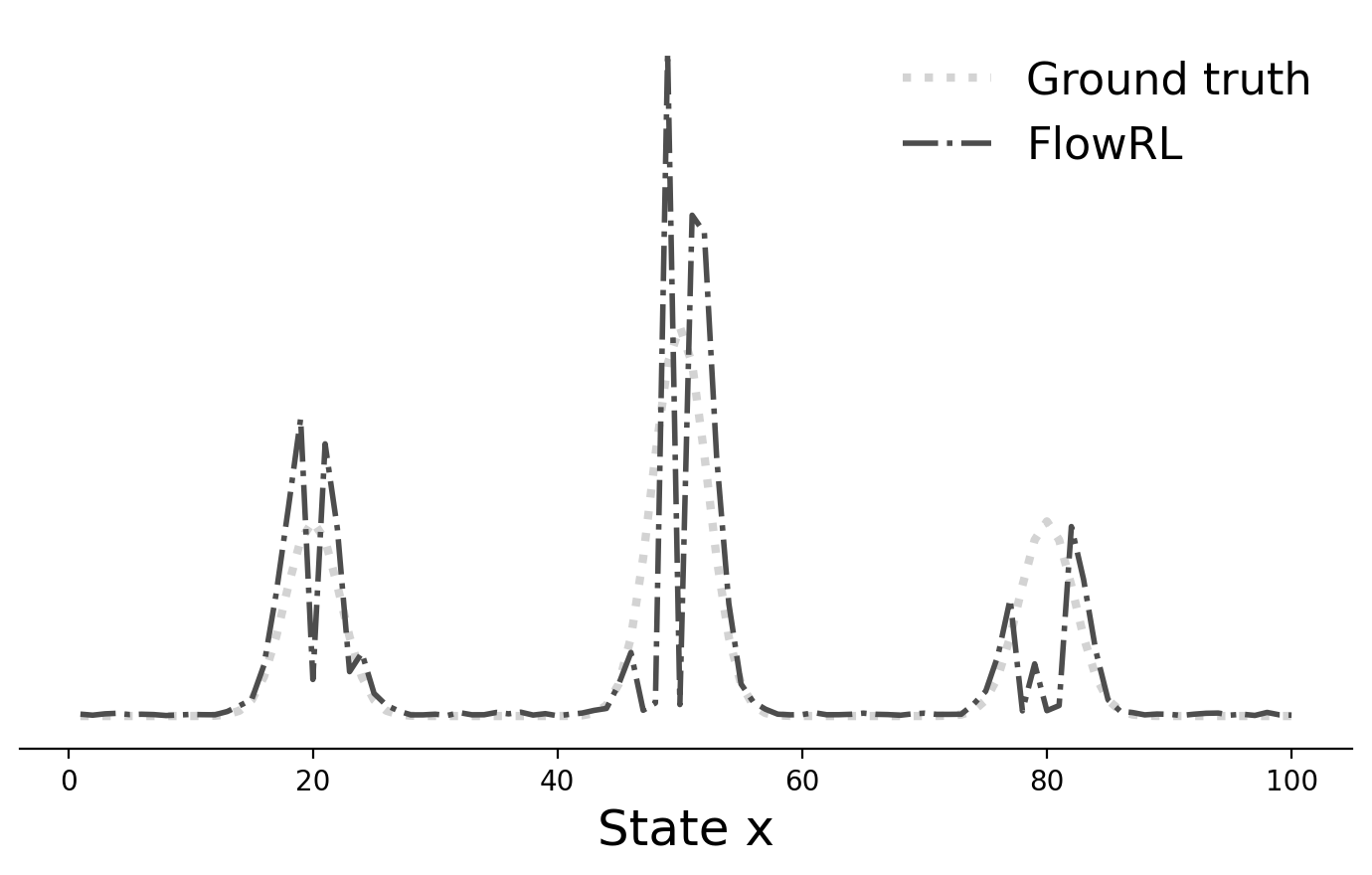}
        \caption{FlowRL}
        \label{fig:flowrl1000}
    \end{subfigure}
    \hfill
    \begin{subfigure}[b]{0.46\textwidth}
        \centering
        \includegraphics[width=\textwidth]{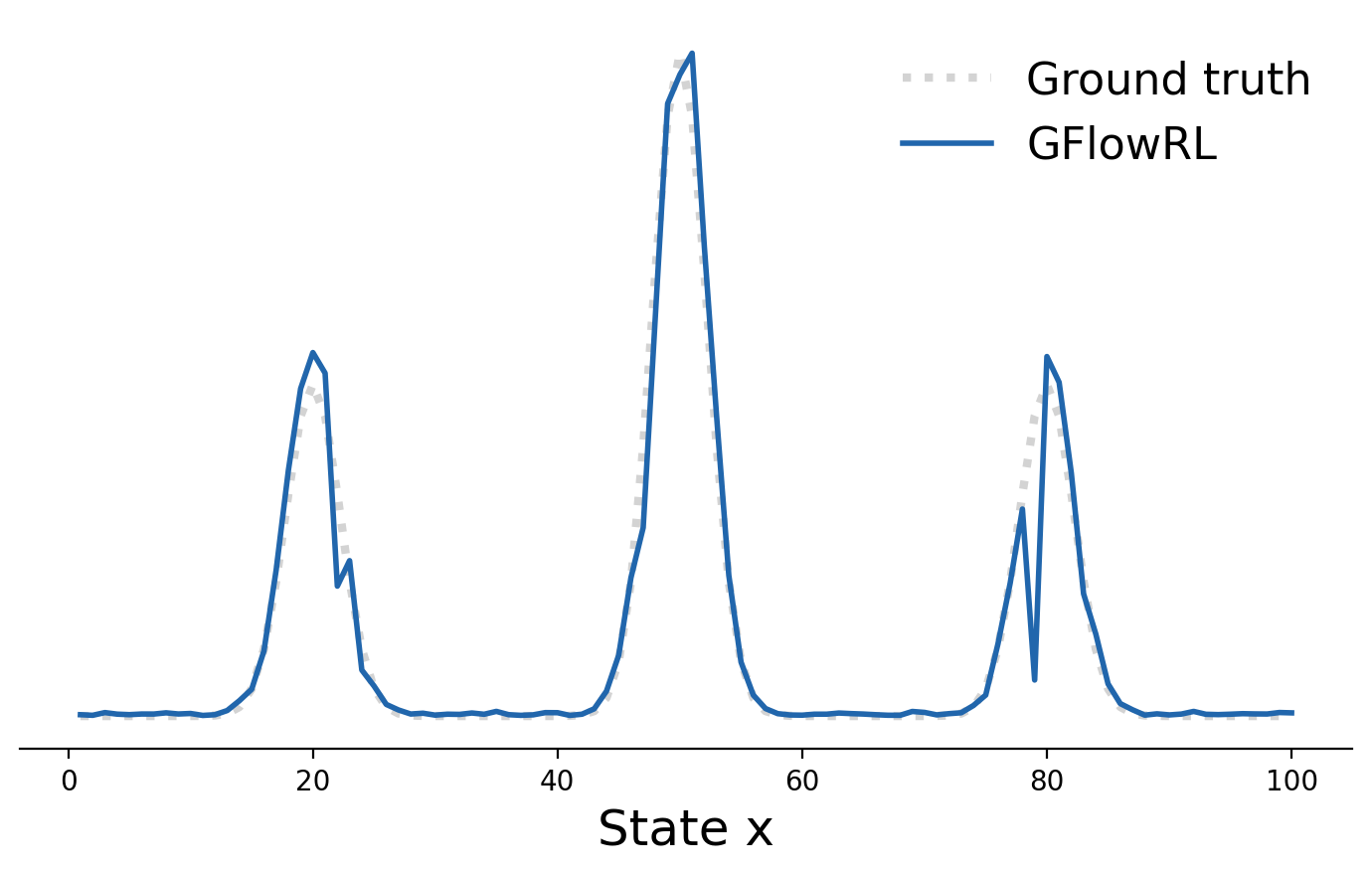}
        \caption{GFlowRL}
        \label{fig:gflowrl1000}
    \end{subfigure}
    \caption{Comparison of distribution matching across GRPO, FlowRL-RandomLogZ, FlowRL, and GFlowRL at the 1000th training step.}
    \label{fig:step1000}
\end{figure}
\begin{figure}[!h]
    \centering
    \begin{subfigure}[b]{0.48\textwidth}
        \centering
        \includegraphics[width=\textwidth]{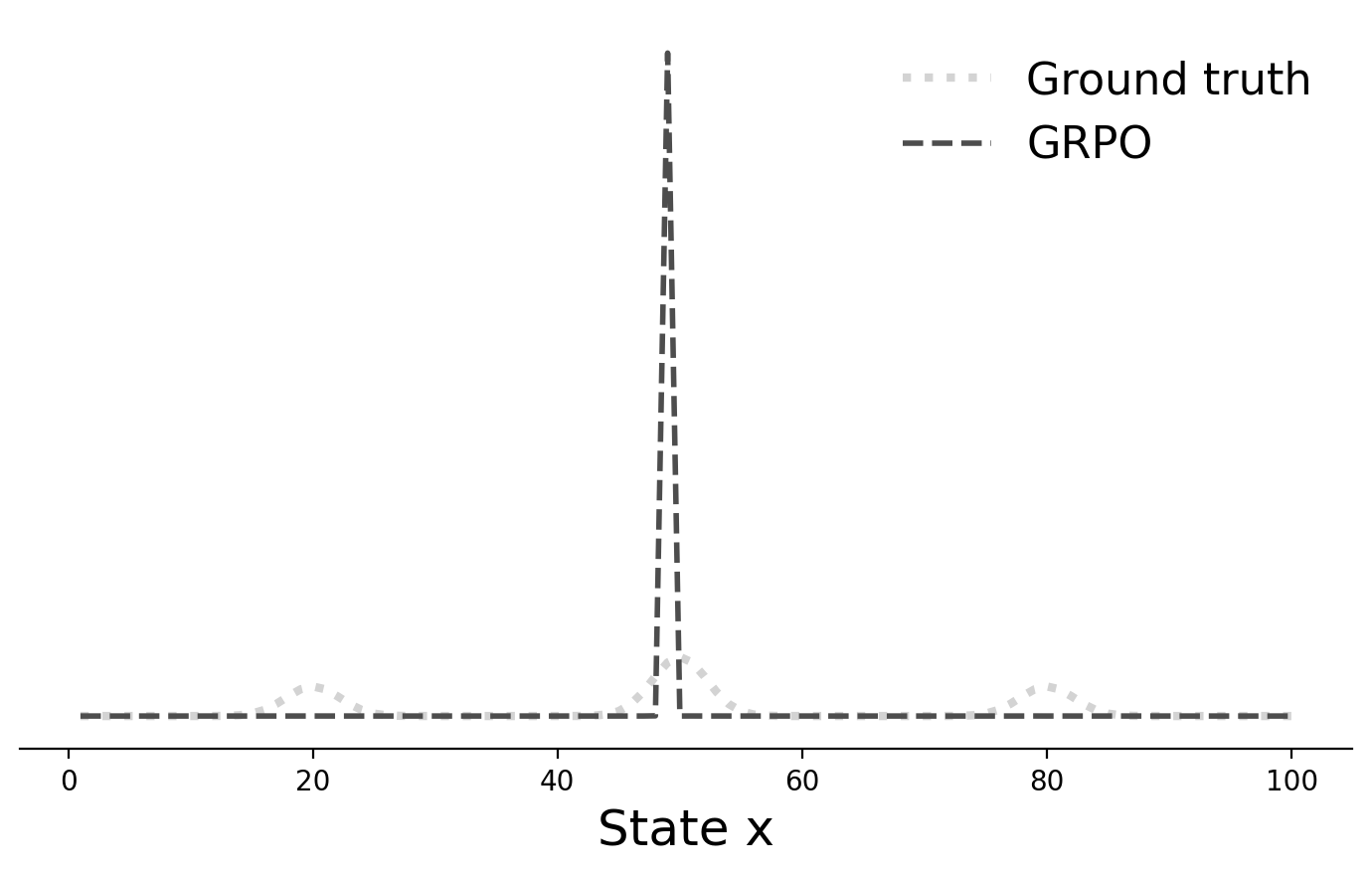}
        \caption{GRPO}
        \label{fig:grpo5000}
    \end{subfigure}
    \hfill
    \begin{subfigure}[b]{0.48\textwidth}
        \centering
        \includegraphics[width=\textwidth]{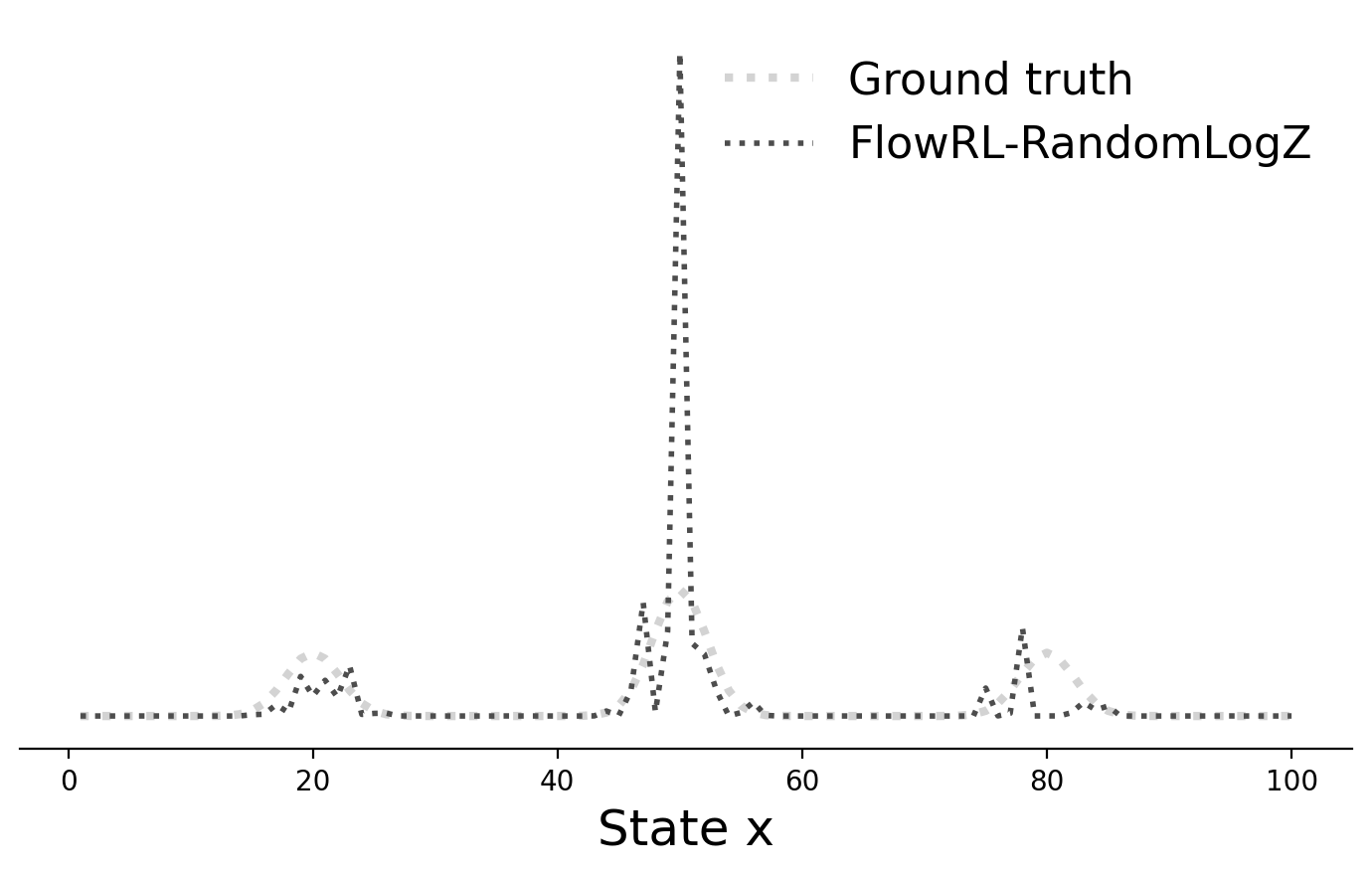}
        \caption{FlowRL-RandomLogZ}
        \label{fig:random5000}
    \end{subfigure}
    \vspace{0.5em} 
    \begin{subfigure}[b]{0.48\textwidth}
        \centering
        \includegraphics[width=\textwidth]{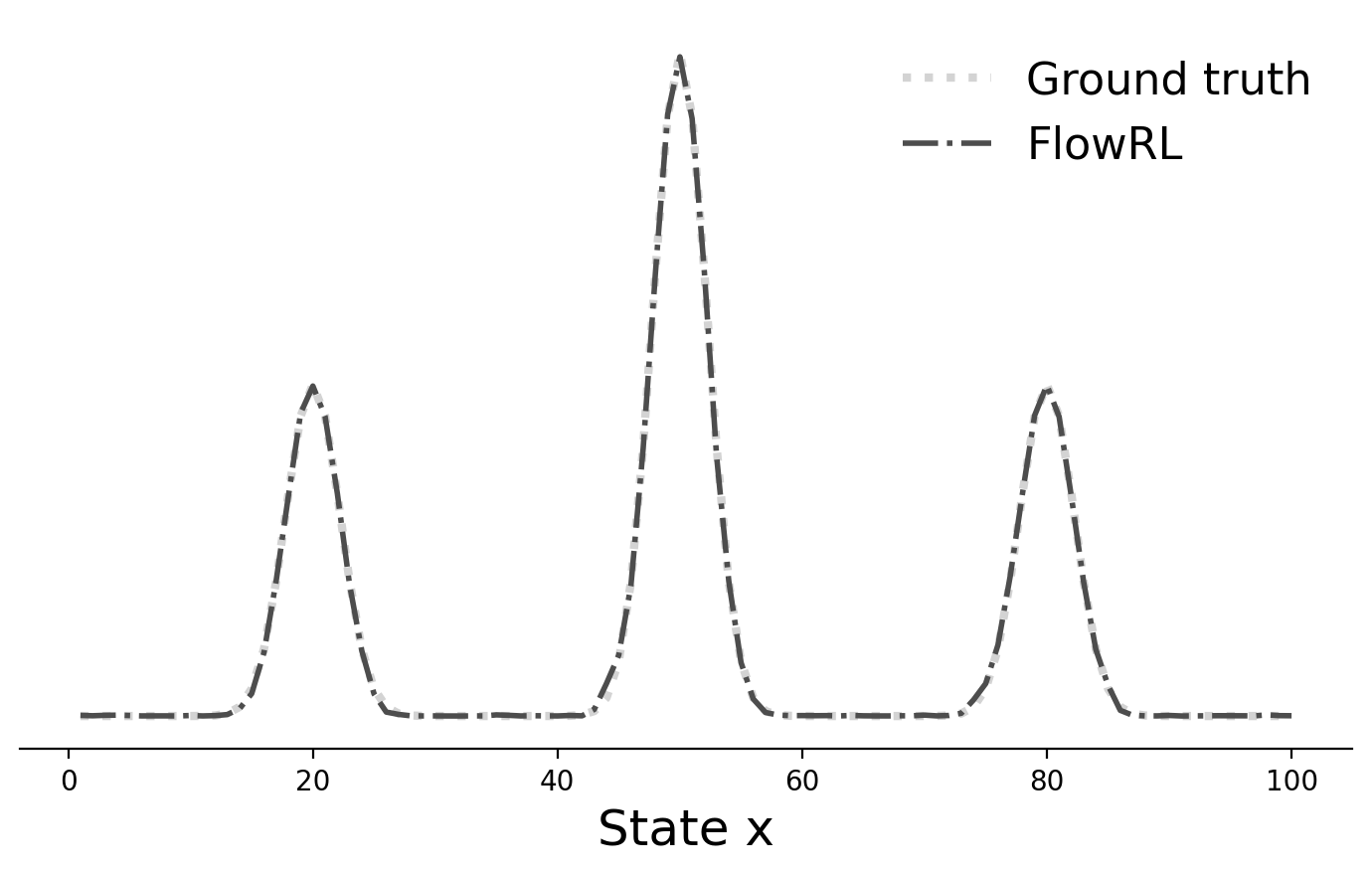}
        \caption{FlowRL}
        \label{fig:flowrl5000}
    \end{subfigure}
    \hfill
    \begin{subfigure}[b]{0.48\textwidth}
        \centering
        \includegraphics[width=\textwidth]{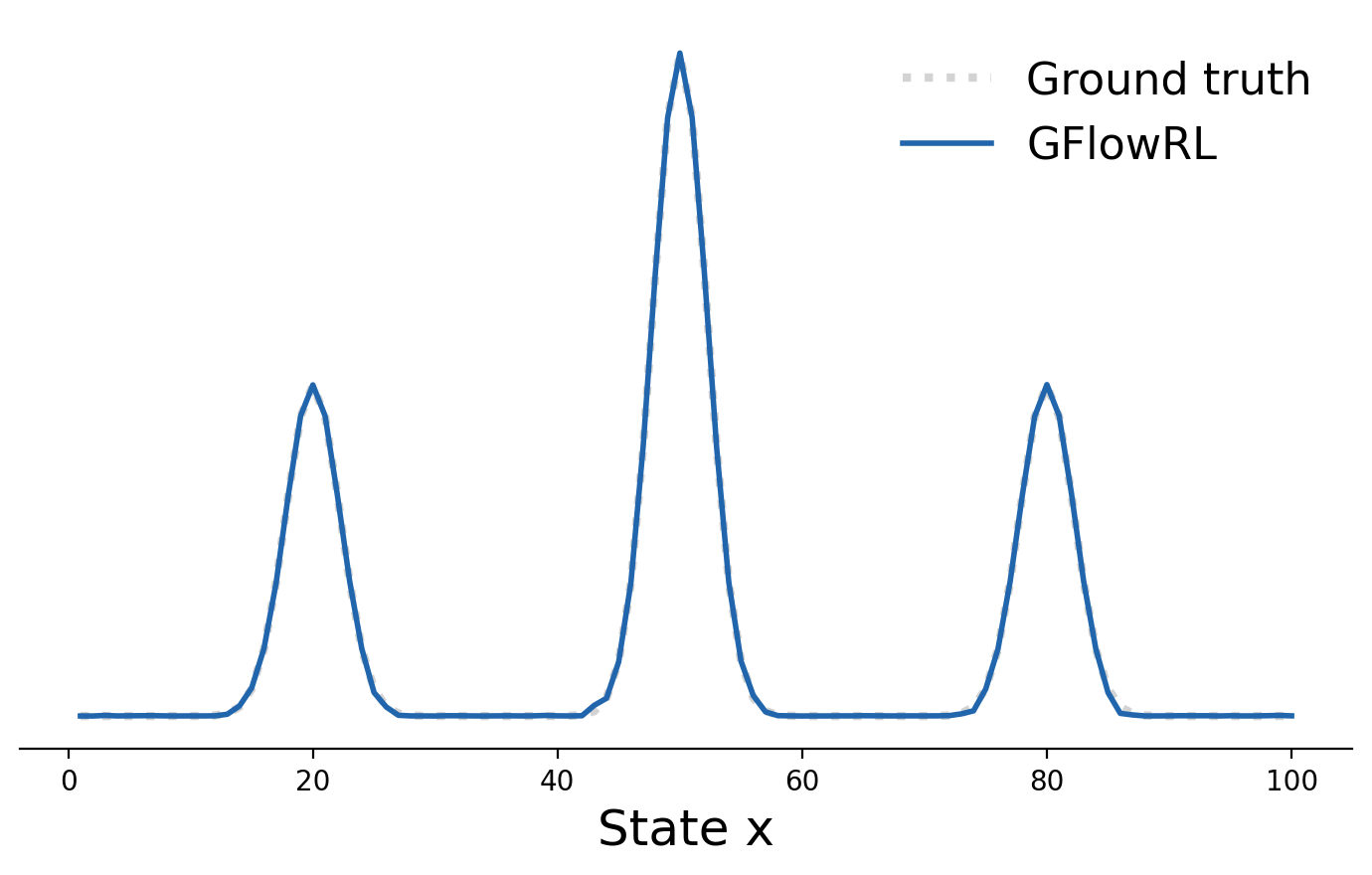}
        \caption{GFlowRL}
        \label{fig:gflowrl5000}
    \end{subfigure}
    \caption{Comparison of distribution matching across GRPO, FlowRL-RandomLogZ, FlowRL, and GFlowRL at the 5000th training step.}
    \label{fig:step5000}
\end{figure}
To evaluate distribution matching performance, we construct a synthetic target distribution defined as a mixture of three Gaussian components with distinct centers. We compare four algorithms: GRPO (a maximum-reward method), FlowRL, FlowRL-RandomLogZ, and GFlowRL. Figures~\ref{fig:step100}, \ref{fig:step500}, \ref{fig:step1000}, and \ref{fig:step5000} illustrate the learned distributions at different training stages (100, 500, 1000, and 5000 steps), respectively.

At the beginning of training (100 steps), GRPO achieves the best distribution fitting, as all methods are still far from convergence. While FlowRL and FlowRL-RandomLogZ produce near-random or uniform samples that do not yet capture the target structure, \method already begins to capture the underlying multimodal shape of the mixture distribution.

As training progresses to 500 steps, a regime comparable to typical LLM post-training settings (e.g., math reasoning), where optimization often spans only a few hundred steps, GRPO begins to exhibit mode collapse, concentrating its probability mass on a single mode. This behavior underscores a key limitation of maximum-reward approaches: although they can rapidly identify high-reward regions, they tend to sacrifice diversity. Meanwhile, FlowRL and FlowRL-RandomLogZ display similar behaviors, indicating that learning the partition function ($log Z$) has not yet had a significant effect. In contrast, \method starts to outperform the other baselines by more effectively distributing probability mass across multiple modes.

\newpage
\section{Qualitative Comparison: \method-14B vs.\ DeepCoder-14B-Preview}
\label{app:qualitative-comparison}

We present a qualitative comparison between \textbf{\method-14B} and \textbf{DeepCoder-14B-Preview}~\citep{deepcoder2025} on Codeforces competition problems.
Both models are fine-tuned from the same base model (DeepSeek-R1-Distill-Qwen-14B).
\method-14B is trained with \method and evaluated at 64k context length; DeepCoder-14B-Preview uses GRPO+ and is also evaluated at 64k context.
All evaluations use Pass@16 with temperature $0.6$ and top-$p = 0.95$.

\paragraph{Aggregate results.}
Table~\ref{tab:cf-per-problem} summarizes the per-problem breakdown across 408 Codeforces problems.
\method-14B exclusively solves 108 problems that DeepCoder cannot, while DeepCoder exclusively solves only~2.

\begin{table}[!h]
\centering
\small
\caption{Per-problem Codeforces comparison (408 problems, Pass@16).}
\label{tab:cf-per-problem}
\begin{tabular}{lrr}
\toprule
& \textbf{Count} & \textbf{\%} \\
\midrule
\method-14B solves, DeepCoder fails & 108 & 26.5 \\
DeepCoder solves, \method-14B fails & 2 & 0.5 \\
Both solve & 149 & 36.5 \\
Neither solves & 149 & 36.5 \\
\bottomrule
\end{tabular}
\end{table}

\paragraph{Selected examples.}
We report representative problems where one model succeeds and the other fails.
For each problem, we report the Avg@16(x) rate for the problem and display a representative valid solution (or the failed attempt). 
\vspace{0.8em}
\hrule
\vspace{0.4em}
\paragraph{Example 1 (CF \#362): \method solves, DeepCoder fails.}
\method: 93\% ($\cmark$) \dotfill DeepCoder: 0\% ($\xmark$)

\smallskip\noindent\textit{Problem.}\enspace You are given a positive integer $k$ and a set $S$ of all integers from $l$ to $r$ (inclusive). You can perform the following two-step operation any number of times (possibly zero): 1. First, choose a number $x$ from the set $S$, such that there are at least $k$ multiples of $x$ in $S$ (includin [...]

\smallskip\noindent{\small\textbf{\method-14B}}
\begin{verbatim}
import sys

t = int(sys.stdin.readline())
for _ in range(t):
    l, r, k = map(int, sys.stdin.readline().split())
    m = r // k
    if m < l:
        print(0)
    else:
        print(m - l + 1)
\end{verbatim}

\noindent{\small\textbf{DeepCoder-14B-Preview}} 

No executable code produced. The model did not converge on a solution.

\vspace{0.8em}
\hrule
\vspace{0.4em}
\paragraph{Example 2 (CF \#68): \method solves, DeepCoder fails.}
\method: 81\% ($\cmark$) \dotfill DeepCoder: 0\% ($\xmark$)

\smallskip\noindent\textit{Problem.}\enspace Turtle and Piggy are playing a number game. First, Turtle will choose an integer $x$, such that $l \le x \le r$, where $l, r$ are given. It's also guaranteed that $2l \le r$. Then, Piggy will keep doing the following operation until $x$ becomes $1$: Choose an integer $p$ such that $p \ge 2$ a [...]

\smallskip\noindent{\small\textbf{\method-14B}}
\begin{verbatim}
primes = [2, 3, 5, 7, 11, 13, 17, 19, 23, 29,
          31, 37, 41, 43, 47, 53, 59, 61, 67, 71]

def solve():
    import sys
    input = sys.stdin.read().split()
    t = int(input[0])
    idx = 1
    for _ in range(t):
        l = int(input[idx])
        r = int(input[idx+1])
        idx += 2

        max_omega = 0

        def dfs(index, current_product, current_omega):
            nonlocal max_omega
            if current_product > r:
                return
  ...
\end{verbatim}

\noindent{\small\textbf{DeepCoder-14B-Preview}}
\begin{verbatim}
   while product * p <= r:
       product *= p
       sum +=1

But this would give the maximum product that is <= r,
composed of the smallest primes, each multiplied as
much as possible.

Wait, no. Because this would multiply each prime once,
not multiple times.

Wait, no. Wait, perhaps I can adjust this:

product = 1
sum =0

for p in primes:
  ...
\end{verbatim}

\vspace{0.8em}
\hrule
\vspace{0.4em}
\paragraph{Example 3 (CF \#156): \method solves, DeepCoder fails.}
\method: 81\% ($\cmark$) \dotfill DeepCoder: 0\% ($\xmark$)

\smallskip\noindent\textit{Problem.}\enspace In his favorite cafe Kmes once again wanted to try the herring under a fur coat. Previously, it would not have been difficult for him to do this, but the cafe recently introduced a new purchasing policy. Now, in order to make a purchase, Kmes needs to solve the following problem: $n$ cards with pri [...]

\smallskip\noindent{\small\textbf{\method-14B}}
\begin{verbatim}
import math

def get_divisors(x):
    divisors = set()
    for i in range(1, int(math.isqrt(x)) + 1):
        if x % i == 0:
            divisors.add(i)
            divisors.add(x // i)
    return divisors

t = int(input())
for _ in range(t):
    n, x = map(int, input().split())
    a = list(map(int, input().split()))
    divisors = get_divisors(x)
    segments = 0
    current_products = set()
    for num in a:
  ...
\end{verbatim}

\noindent{\small\textbf{DeepCoder-14B-Preview}}
\begin{verbatim}
for each a_i in a:
   temp = set()
   temp.add(a_i)  # subset consisting of only a_i

   for p in S:
       product = p * a_i
       if product > x:
           continue
       temp.add(product)

   # Now, the new possible products are S union temp
  ...
\end{verbatim}

\vspace{0.8em}
\hrule
\vspace{0.4em}
\paragraph{Example 4 (CF \#171): \method solves, DeepCoder fails.}
\method: 37\% ($\cmark$) \dotfill DeepCoder: 0\% ($\xmark$)

\smallskip\noindent\textit{Problem.}\enspace After winning another Bed Wars game, Masha and Olya wanted to relax and decided to play a new game. Masha gives Olya an array $a$ of length $n$ and a number $s$. Now Olya's task is to find a non-negative number $x$ such that $\displaystyle\sum_{i=1}^{n} a_i \oplus x = s$. But she is very tired after [...]

\smallskip\noindent{\small\textbf{\method-14B}}
\begin{verbatim}
import sys
from collections import defaultdict

def main():
    input = sys.stdin.read().split()
    ptr = 0
    t = int(input[ptr])
    ptr += 1
    for _ in range(t):
        n = int(input[ptr])
        k = int(input[ptr+1])
        ptr +=2
        s = input[ptr]
        ptr +=1
        a_list = []
        for _ in range(n):
            a_str = input[ptr]
            ptr +=1
  ...
\end{verbatim}

\noindent{\small\textbf{DeepCoder-14B-Preview}} 

No executable code produced. The model did not converge on a solution.

\vspace{0.8em}
\hrule
\vspace{0.4em}
\paragraph{Example 5 (CF \#375): \method solves, DeepCoder fails.}
\method: 25\% ($\cmark$) \dotfill DeepCoder: 0\% ($\xmark$)

\smallskip\noindent\textit{Problem.}\enspace Suppose we partition the elements of an array $b$ into any number $k$ of non-empty multisets $S_1, S_2, \ldots, S_k$, where $k$ is an arbitrary positive integer. Define the score of $b$ as the maximum value of $\operatorname{MEX}(S_1)^{\text{*}} + \operatorname{MEX}(S_2) + \ldots$~[\ldots]

\smallskip\noindent{\small\textbf{\method-14B}}
\begin{verbatim}
import sys
MOD = 998244353

def main():
    input = sys.stdin.read
    data = input().split()
    idx = 0
    t = int(data[idx])
    idx += 1
    for _ in range(t):
        n = int(data[idx])
        idx += 1
        a = list(map(int, data[idx:idx+n]))
        idx += n

        cnt = [0] * n
        for num in a:
            if num < n:
  ...
\end{verbatim}

\noindent{\small\textbf{DeepCoder-14B-Preview}}

No executable code produced.
The model did not converge on a solution.

\vspace{0.8em}
\hrule
\vspace{0.4em}
\paragraph{Example 6 (CF \#75): DeepCoder solves, \method fails.}
\method: 0\% ($\xmark$) \dotfill DeepCoder: 6\% ($\cmark$)

\smallskip\noindent\textit{Problem.}\enspace Recently, the first-year student Maxim learned about the Collatz conjecture, but he didn't pay much attention during the lecture, so he believes that the following process is mentioned in the conjecture: There is a variable $x$ and a constant $y$. The following operation is performed $k$ times: [...]

\smallskip\noindent{\small\textbf{\method-14B}}
\begin{verbatim}
t = int(input())
for _ in range(t):
    x, y, k = map(int, input().split())
    while k > 0:
        if x < y - 1:
            break
        temp = x + 1
        while temp % y == 0:
            temp //= y
        if temp == x:
            break
        x = temp
        k -= 1
    x += k
    print(x)
\end{verbatim}

\noindent{\small\textbf{DeepCoder-14B-Preview}}
\begin{verbatim}
t = int(input())
for _ in range(t):
    x, y, k = map(int, input().split())
    while k > 0:
        if x == 1 and y == 2:
            break
        r = x % y
        if r == y - 1:
            # Process one step where x+1 is divisible by y
            temp = x + 1
            while temp % y == 0:
                temp //= y
            x = temp
            k -= 1
        else:
            s = (y - r - 1)
            if s >= k:
                x += k
  ...
\end{verbatim}

\vspace{0.8em}
\hrule

\end{document}